\documentclass[accepted]{uai2026} %

\usepackage[american]{babel}

\usepackage{natbib}
\usepackage{microtype}
\usepackage{graphicx}
\usepackage{booktabs}
\usepackage{cancel}
\usepackage{adjustbox,array}

\hypersetup{
  colorlinks,
  linkcolor={red!50!black},
  citecolor={blue!50!black},
  urlcolor={blue!80!black}
}

\usepackage{amsmath}
\usepackage{amssymb}
\usepackage{mathtools}
\usepackage{amsthm}
\usepackage{thmtools, thm-restate}
\usepackage{svg}
\usepackage{wrapfig}
\usepackage{caption}
\usepackage{subcaption}
\usepackage{float}
\usepackage{comment}
\usepackage[ruled,vlined]{algorithm2e}
\usepackage{algpseudocode}

\theoremstyle{plain}
\newtheorem{theorem}{Theorem}[section]

\theoremstyle{definition}
\newtheorem{definition}{Definition}[section]
\newtheorem{assum}{Assumption}[section]

\theoremstyle{remark}
\newtheorem{remark}{Remark}[section]

\usepackage[textsize=tiny]{todonotes}

\usepackage{url}

\usepackage{breakurl}

\usepackage{makecell}
\usepackage{etoc} %

\usepackage[capitalize,noabbrev,nameinlink]{cleveref}
\usepackage{crossreftools}

\crefname{equation}{}{}
\crefname{section}{\S}{\S}
\crefname{subsection}{\S}{\S}
\crefname{subsubsection}{\S}{\S}
\crefname{figure}{Fig.}{Figs.}
\crefname{prop}{Prop.}{Props.}
\crefname{proposition}{Prop.}{Props.}
\crefname{appendix}{Appx.}{Appxs.}
\crefname{algorithm}{Alg.}{Algs.}
\crefname{theorem}{Thm.}{Thms.}
\crefname{conjecture}{Conj.}{Conjs.}
\crefname{researchquestion}{Q.}{Qs.}
\crefname{definition}{Defn.}{Defns.}
\crefname{corollary}{Cor.}{Cors.}
\crefname{lem}{Lem.}{Lems.}
\crefname{table}{Tab.}{Tabs.}
\crefname{assum}{Assum.}{Assums.}
\crefname{example}{Ex.}{Exs.}

\newenvironment{talign*}
 {\let\displaystyle\textstyle\csname align*\endcsname}
 {\endalign}
\newenvironment{talign}
 {\let\displaystyle\textstyle\csname align\endcsname}
 {\endalign}

\newcommand{\textfrac}[2]{{\textstyle\frac{#1}{#2}}}

\def\reals{\mathbb{R}}

\newcommand{\eps}{\varepsilon}

\newcommand{\sstext}[1]{\ \ \text{#1}\ \ }
\newcommand{\E}{\mathbb{E}}
\newcommand{\Var}{\mathrm{Var}}
\newcommand{\Cov}{\mathrm{Cov}}
\newcommand{\parent}{\textup{Pa}}

\newcommand{\defeq}{\triangleq}
\newcommand{\half}{\frac{1}{2}}
\def\indep{\perp\!\!\!\perp}
\newcommand{\inner}[2]{\langle{#1},{#2}\rangle}

\newcommand{\ie}{i.e.\@\xspace}

\newcommand{\cf}{cf.\@\xspace} %

\newcommand{\inv}[1]{\ensuremath{#1^{-1}}}

\newcommand{\diag}[1]{\ensuremath{\mathrm{diag}\parenthesis{#1}}}

\newcommand{\mat}[1]{\ensuremath{\boldsymbol{\mathrm{#1}}}}

\newcommand{\norm}[1]{\ensuremath{\left\Vert#1\right\Vert}}

\newcommand{\abs}[1]{\ensuremath{\left|#1\right|}}
\newcommand{\logabsdet}[1]{\ensuremath{\log\left|\det#1\right|}}

\newcommand{\expectation}[1]{\ensuremath{\mathbb{E}_{#1}}}
\newcommand{\rr}[1]{\ensuremath{\mathbb{R}^{#1}}}

\newcommand{\parenthesis}[1]{\ensuremath{\left(#1\right)}}
\newcommand{\brackets}[1]{\ensuremath{\left[#1\right]}}

\usepackage[acronym, automake, toc, nomain, nopostdot, style=tree, nonumberlist,numberedsection]{glossaries}
\usepackage{glossary-mcols}
\usepackage{xparse}
\usepackage{xspace}
\setglossarystyle{mcolindex}

\newglossary{abbrev}{abs}{abo}{Nomenclature}

\newglossaryentry{aux}{
    name        = \ensuremath{\mathrm{\boldsymbol{u}}} ,
    description = {auxiliary variable} ,
    type        = abbrev,
}

\newglossaryentry{im}{
    name        = \ensuremath{\mathrm{Im}} ,
    description = {image space} ,
    type        = abbrev,
}

\newglossaryentry{ker}{
    name        = \ensuremath{\mathrm{Ker}} ,
    description = {kernel space} ,
    type        = abbrev,
}

\newglossaryentry{kronecker}{
    name        = \ensuremath{\otimes} ,
    description = {Kronecker product} ,
    type        = abbrev,
}

\newglossaryentry{loss}{
    name        = \ensuremath{\mathcal{L}} ,
    description = {loss function} ,
    type        = abbrev,
}

\newglossaryentry{numenv}{
    name        = \ensuremath{\abs{E}} ,
    description = {number of environments} ,
    type        = abbrev,
}

\newglossaryentry{lr}{
    name        = \ensuremath{\eta} ,
    description = {learning rate} ,
    type        = abbrev,
}

\newglossaryentry{hypersphere}{
    name        = \ensuremath{\mathcal{S}} ,
    description = {hypersphere} ,
    type        = abbrev,
}

\newglossaryentry{dec}{
    name        = \ensuremath{\boldsymbol{f}} ,
    description = {decoder map $\gls{Latent}\to\gls{Obs}$} ,
    type        = abbrev,
}

\newglossaryentry{deccomp}{
    name        = \ensuremath{f} ,
    description = {decoder map component} ,
    type        = abbrev,
}

\newglossaryentry{enc}{
    name        = \ensuremath{\boldsymbol{g}} ,
    description = {encoder map $\gls{Obs}\to\gls{Latent}$} ,
    type        = abbrev,
}

\newglossaryentry{numdata}{
    name        = \ensuremath{n} ,
    description = {number of samples} ,
    type        = abbrev,
}

\newglossaryentry{observations}{type=abbrev,name=Observations,description={\nopostdesc}}

\newglossaryentry{obs}{
    name        = \ensuremath{\boldsymbol{x}} ,
    description = {observation vector} ,
    type        = abbrev,
    parent      = observations,
}

\newglossaryentry{obscomp}{
    name        = \ensuremath{x} ,
    description = {observation single component} ,
    type        = abbrev,
    parent      = observations,
}

\newglossaryentry{Obs}{
    name        = \ensuremath{\mathcal{X}} ,
    description = {observation space} ,
    type        = abbrev,
    parent      = observations,
}

\newglossaryentry{obsdim}{
    name        = \ensuremath{D} ,
    description = {dimensionality of the observation space \gls{Obs}} ,
    type        = abbrev,
    parent      = observations,
}

\newglossaryentry{obsmat}{
    name        = \ensuremath{\mat{X}} ,
    description = {observation matrix of \rr{\gls{numdata}\times\gls{obsdim}}} ,
    type        = abbrev,
    parent      = observations,
}

\newglossaryentry{obspos}{
    name        = \ensuremath{\tilde{\boldsymbol{x}}} ,
    description = {positive observation vector} ,
    type        = abbrev,
    parent      = observations,
}

\newglossaryentry{obsneg}{
    name        = \ensuremath{{\boldsymbol{x}}^{-}} ,
    description = {negative observation vector} ,
    type        = abbrev,
    parent      = observations,
}

\newglossaryentry{labels}{type=abbrev,name=Labels,description={\nopostdesc}}

\newglossaryentry{label}{
    name        = \ensuremath{\boldsymbol{y}} ,
    description = {label vector} ,
    type        = abbrev,
    parent      = labels,
}

\newglossaryentry{labelhat}{
    name        = \ensuremath{\widehat{\boldsymbol{y}}} ,
    description = {estimated label vector} ,
    type        = abbrev,
    parent      = labels,
}

\newglossaryentry{labelcomp}{
    name        = \ensuremath{y} ,
    description = {label component} ,
    type        = abbrev,
    parent      = labels,
}

\newglossaryentry{labelcomphat}{
    name        = \ensuremath{\widehat{y}} ,
    description = {label component} ,
    type        = abbrev,
    parent      = labels,
}

\newglossaryentry{labelset}{
    name        = \ensuremath{\mathcal{Y}} ,
    description = {label set} ,
    type        = abbrev,
    parent      = labels,
}

\newglossaryentry{labeldim}{
    name        = \ensuremath{C} ,
    description = {number of classes in the label set \gls{labelset}} ,
    type        = abbrev,
    parent      = labels,
}

\newglossaryentry{latents}{type=abbrev,name=Latents,description={\nopostdesc}}

\newglossaryentry{latent}{
    name        = \ensuremath{\boldsymbol{z}} ,
    description = {latent vector} ,
    type        = abbrev,
    parent     = latents,
}

\newglossaryentry{latentcomp}{
    name        = \ensuremath{z} ,
    description = {latent single component} ,
    type        = abbrev,
    parent     = latents,
}

\newglossaryentry{Latent}{
    name        = \ensuremath{\mathcal{Z}} ,
    description = {latents} ,
    type        = abbrev,
    parent     = latents,
}

\newglossaryentry{latentdim}{
    name        = \ensuremath{d} ,
    description = {dimensionality of the latent space \gls{Latent}} ,
    type        = abbrev,
    parent     = latents,
}

\newglossaryentry{latentmat}{
    name        = \ensuremath{\mat{Z}} ,
    description = {latent matrix of \rr{\gls{numdata}\times\gls{latentdim}}} ,
    type        = abbrev,
    parent      = latents,
}

\newglossaryentry{latentpos}{
    name        = \ensuremath{\tilde{\boldsymbol{z}}} ,
    description = {positive latent vector} ,
    type        = abbrev,
    parent      = latents,
}

\newglossaryentry{latentneg}{
    name        = \ensuremath{\boldsymbol{z}^{-}} ,
    description = {negative latent vector} ,
    type        = abbrev,
    parent      = observations,
}

\newglossaryentry{sigmaz}{
    name        = \ensuremath{\boldsymbol{\sigma}_{\gls{latentcomp}}} ,
    description = {std of \gls{latentcomp}} ,
    type        = abbrev,
    parent     = latents,
}

\newglossaryentry{content}{
    name        = \ensuremath{\boldsymbol{z}^{c}} ,
    description = {content latent vector} ,
    type        = abbrev,
    parent     = latents,
}

\newglossaryentry{contentcomp}{
    name        = \ensuremath{z^{c}} ,
    description = {content latent single component} ,
    type        = abbrev,
    parent     = latents,
}

\newglossaryentry{Content}{
    name        = \ensuremath{\mathcal{Z}^{c}} ,
    description = {content} ,
    type        = abbrev,
    parent     = latents,
}

\newglossaryentry{contentdim}{
    name        = \ensuremath{d_{c}} ,
    description = {dimensionality of \gls{content}} ,
    type        = abbrev,
    parent     = latents,
}

\newglossaryentry{sigmac}{
    name        = \ensuremath{\boldsymbol{\sigma}_{c}} ,
    description = {std of \gls{contentcomp}} ,
    type        = abbrev,
    parent     = latents,
}

\newglossaryentry{style}{
    name        = \ensuremath{\boldsymbol{z}^{s}} ,
    description = {style latent vector} ,
    type        = abbrev,
    parent     = latents,
}

\newglossaryentry{stylecomp}{
    name        = \ensuremath{z^{s}} ,
    description = {style latent single component} ,
    type        = abbrev,
    parent     = latents,
}

\newglossaryentry{Style}{
    name        = \ensuremath{\mathcal{Z}^{s}} ,
    description = {style} ,
    type        = abbrev,
    parent     = latents,
}

\newglossaryentry{styledim}{
    name        = \ensuremath{d_{s}} ,
    description = {dimensionality of \gls{style}} ,
    type        = abbrev,
    parent     = latents,
}

\newglossaryentry{sigmas}{
    name        = \ensuremath{\boldsymbol{\sigma}_{s}} ,
    description = {std of \gls{stylecomp}} ,
    type        = abbrev,
    parent     = latents,
}

\newglossaryentry{modality}{
    name        = \ensuremath{\boldsymbol{z}^{m}} ,
    description = {modality-specific latent vector} ,
    type        = abbrev,
    parent     = latents,
}

\newglossaryentry{modalitycomp}{
    name        = \ensuremath{z^{m}} ,
    description = {modality-specific  latent single component} ,
    type        = abbrev,
    parent     = latents,
}

\newglossaryentry{Modality}{
    name        = \ensuremath{\mathcal{Z}^{m}} ,
    description = {latent subspace of \gls{modality}} ,
    type        = abbrev,
    parent     = latents,
}

\newglossaryentry{modalitydim}{
    name        = \ensuremath{d_{m}} ,
    description = {dimensionality of \gls{modality}} ,
    type        = abbrev,
    parent     = latents,
}

\newglossaryentry{algebra}{type=abbrev,name=Algebra,description={\nopostdesc}}

\newglossaryentry{identity}{
    name        = \ensuremath{\boldsymbol{\mathrm{I}}} ,
    description = { identity matrix} ,
    type        = abbrev,
    parent      = algebra,
}
\newcommand{\Id}[1]{\ensuremath{\gls{identity}_{#1}}}

\newglossaryentry{ones}{
    name        = \ensuremath{\boldsymbol{\mathrm{1}}} ,
    description = {a vector of ones} ,
    type        = abbrev,
    parent      = algebra,
}

\newglossaryentry{zeros}{
    name        = \ensuremath{\boldsymbol{\mathrm{0}}} ,
    description = {a vector of zeros} ,
    type        = abbrev,
    parent      = algebra,
}

\newglossaryentry{jacobian}{
    name        = \ensuremath{\boldsymbol{\mathrm{J}}} ,
    description = {Jacobian matrix} ,
    type        = abbrev,
    parent      = algebra,
}

\newglossaryentry{hessian}{
    name        = \ensuremath{\boldsymbol{\mathrm{H}}} ,
    description = {Hessian matrix} ,
    type        = abbrev,
    parent      = algebra,
}

\newglossaryentry{d}{
    name        = \ensuremath{\boldsymbol{\mathrm{D}}} ,
    description = {diagonal matrix} ,
    type        = abbrev,
    parent      = algebra,
}

\newglossaryentry{o}{
    name        = \ensuremath{\boldsymbol{\mathrm{O}}},
    description = {orthogonal matrix} ,
    type        = abbrev,
    parent      = algebra,
}

\newglossaryentry{scalar}{
    name        = \ensuremath{\alpha} ,
    description = {scalar field} ,
    type        = abbrev,
    parent      = algebra,
}

\newglossaryentry{perm}{
    name        = \ensuremath{\mathbb{P}} ,
    description = {group of permutation matrices} ,
    type        = abbrev,
    parent      = algebra,
}

\newglossaryentry{p}{
    name        = \ensuremath{\mat{P}},
    description = {permutation matrix} ,
    type        = abbrev,
    parent      = algebra,
}

\newglossaryentry{prob}{type=abbrev,name=Probability theory,description={\nopostdesc}}

\newglossaryentry{cov}{
    name        = \ensuremath{\boldsymbol{\mathrm{\Sigma}}},
    description = {covariance matrix} ,
    type        = abbrev,
    parent      = prob,
}

\newglossaryentry{mean}{
    name        = \ensuremath{\boldsymbol{\mu}},
    description = {mean} ,
    type        = abbrev,
    parent      = prob,
}

\newglossaryentry{std}{
    name        = \ensuremath{\boldsymbol{\sigma}},
    description = {standard deviation} ,
    type        = abbrev,
    parent      = prob,
}

\newglossaryentry{entropy}{
    name        = \ensuremath{\mathrm{H}} ,
    description = {entropy} ,
    type        = abbrev,
    parent      = prob,
}

\newglossaryentry{expfamparam}{
    name        = \ensuremath{\boldsymbol{\theta}} ,
    description = {parameter of exponential family} ,
    type        = abbrev,
    parent      = prob,
}

\newglossaryentry{expfamnatparam}{
    name        = \ensuremath{\boldsymbol{\eta}} ,
    description = {natural parameter of exponential family} ,
    type        = abbrev,
    parent      = prob,
}

\newglossaryentry{expfamsuffstat}{
    name        = \ensuremath{T(\gls{obs})} ,
    description = {sufficient statistics of exponential family} ,
    type        = abbrev,
    parent      = prob,
}

\newglossaryentry{expfamlogpartition}{
    name        = \ensuremath{A} ,
    description = {log parition function of exponential family (depends on \gls{expfamnatparam})} ,
    type        = abbrev,
    parent      = prob,
}

\newglossaryentry{wishart}{
    name        = \ensuremath{\mathcal{W}} ,
    description = {Wishart distribution} ,
    type        = abbrev,
    parent      = prob,
}

\newglossaryentry{normal}{
    name        = \ensuremath{\mathcal{N}} ,
    description = {normal distribution} ,
    type        = abbrev,
    parent      = prob,
}

\newglossaryentry{matrixnormal}{
    name        = \ensuremath{\mathcal{MN}} ,
    description = {normal distribution} ,
    type        = abbrev,
    parent      = prob,
}

\newglossaryentry{causal}{type=abbrev,name=Causality,description={\nopostdesc}}

\newglossaryentry{cause}{
    name        = \ensuremath{\boldsymbol{N}},
    description = {noise (independent)  variable vector} ,
    type        = abbrev,
    parent      = causal,
}

\newglossaryentry{causecomp}{
    name        = \ensuremath{N},
    description = {noise (independent)  variable component} ,
    type        = abbrev,
    parent      = causal,
}

\newglossaryentry{Cause}{
    name        = \ensuremath{\mathcal{N}} ,
    description = {space of the noise variables} ,
    type        = abbrev,
    parent      = causal,
}

\newglossaryentry{effect}{
    name        = \ensuremath{\boldsymbol{X}},
    description = {observation vector} ,
    type        = abbrev,
    parent      = causal,
}
\newglossaryentry{effectcomp}{
    name        = \ensuremath{X},
    description = {observation component} ,
    type        = abbrev,
    parent      = causal,
}

\newglossaryentry{Effect}{
    name        = \ensuremath{\mathcal{X}} ,
    description = {space of the effect variables} ,
    type        = abbrev,
    parent      = causal,
}

\newglossaryentry{pa}{
    name        = \ensuremath{\boldsymbol{Pa}},
    description = {parents of \gls{effect}} ,
    type        = abbrev,
    parent      = causal,
}

\newglossaryentry{nondesc}{
    name        = \ensuremath{\boldsymbol{ND}},
    description = {non-descendants of \gls{effect}} ,
    type        = abbrev,
    parent      = causal,
}

\newglossaryentry{nondescminuspa}{
    name        = \ensuremath{\boldsymbol{\overline{ND}}},
    description = {non-descendants of \gls{effect}, excluding its parents} ,
    type        = abbrev,
    parent      = causal,
}

\newglossaryentry{semf}{
    name        = \ensuremath{\boldsymbol{f}},
    description = {structural assignment in \glspl{sem}} ,
    type        = abbrev,
    parent      = causal,
}

\newglossaryentry{semfcomp}{
    name        = \ensuremath{f},
    description = {a component of \gls{semf}} ,
    type        = abbrev,
    parent      = causal,
}

\newglossaryentry{order}{
    name        = \ensuremath{\pi},
    description = {causal ordering} ,
    type        = abbrev,
    parent      = causal,
}

\newglossaryentry{indexset}{
    name        = \ensuremath{\mathcal{I}},
    description = {index set} ,
    type        = abbrev,
    parent      = causal,
}
\newglossaryentry{adjacency}{
    name        = \ensuremath{\boldsymbol{\mathcal{A}}} ,
    description = {adjacency matrix of a \glspl{sem}} ,
    type        = abbrev,
    parent      = causal,
}

\newglossaryentry{connectivity}{
    name        = \ensuremath{\boldsymbol{\mathcal{C}}} ,
    description = {connectivity matrix of a \glspl{sem}} ,
    type        = abbrev,
    parent      = causal,
}

\newglossaryentry{dependency}{
    name        = \ensuremath{\mathcal{D}} ,
    description = {dependency matrix of a \glspl{sem}} ,
    type        = abbrev,
    parent      = causal,
}

\newglossaryentry{seq}{
    name        = \ensuremath{\sim_{\acrshort{dag}}} ,
    description = {structural equivalence} ,
    type        = abbrev,
    parent      = causal,
}

\newglossaryentry{contrastive}{type=abbrev,name=Contrastive Learning,description={\nopostdesc}}
\newglossaryentry{clloss}{
    name        = \ensuremath{\mathcal{L}_{\mathrm{\acrshort{cl}}}} ,
    description = {contrastive loss function} ,
    type        = abbrev,
    parent      = contrastive,
}

\newglossaryentry{alignloss}{
    name        = \ensuremath{\mathcal{L}_{\mathrm{align}}} ,
    description = {alignment term in \gls{clloss}} ,
    type        = abbrev,
    parent      = contrastive,
}

\newglossaryentry{uniformloss}{
    name        = \ensuremath{\mathcal{L}_{\mathrm{uniform}}} ,
    description = {uniformity term in \gls{clloss}} ,
    type        = abbrev,
    parent      = contrastive,
}

\newglossaryentry{temp}{
    name        = \ensuremath{{\boldsymbol{\tau}}} ,
    description = {temperature in \gls{clloss}} ,
    type        = abbrev,
    parent      = contrastive,
}

\newglossaryentry{numneg}{
    name        = \ensuremath{M} ,
    description = {number of negative samples} ,
    type        = abbrev,
    parent      = contrastive,
}

\newglossaryentry{vaes}{type=abbrev,name=\acrlongpl{vae},description={\nopostdesc}}

\newglossaryentry{q}{
    name        = \ensuremath{q_{\gls{encpar}}(\gls{latent}|\gls{obs})} ,
    description = {variational posterior of the \acrshort{vae}, mapping $\gls{obs}\mapsto\gls{latent}$ parametrized by \gls{encpar}} ,
    type        = abbrev,
    parent      = vaes,
}
\newglossaryentry{qopt}{
    name        = \ensuremath{q_{\widehat{\gls{encpar}}}(\gls{latent}|\gls{obs})} ,
    description = {optimal variational posterior of the \acrshort{vae}, mapping $\gls{obs}\mapsto\gls{latent}$ parametrized by \gls{encpar}} ,
    type        = abbrev,
    parent      = vaes,
}

\newglossaryentry{encpar}{
    name        = \ensuremath{\boldsymbol{\phi}} ,
    description = {parameters of the variational posterior \gls{q}} ,
    type        = abbrev,
    parent      = vaes,
}

\newglossaryentry{encparopt}{
    name        = \ensuremath{\widehat{\boldsymbol{\phi}}} ,
    description = {optimal parameters of the variational posterior \gls{q}} ,
    type        = abbrev,
    parent      = vaes,
}

\newglossaryentry{var_family}{
    name        = \ensuremath{\mathcal{Q}} ,
    description = {distribution family of the variational posterior \gls{q} } ,
    type        = abbrev,
    parent      = vaes,
}

\newglossaryentry{pz}{
    name        = \ensuremath{p_0(\gls{latent})} ,
    description = {latent prior distribution} ,
    type        = abbrev,
    parent      = vaes,
}

\newglossaryentry{px}{
    name        = \ensuremath{p_{\gls{decpar}}(\gls{obs})} ,
    description = {marginal likelihood } ,
    type        = abbrev,
    parent      = vaes,
}

\newglossaryentry{pdata}{
    name        = \ensuremath{p(\gls{obs})} ,
    description = {data distribution } ,
    type        = abbrev,
    parent      = vaes,
}

\newglossaryentry{mean_enc}{
    name        = \ensuremath{\mu_{\gls{latent}|\gls{obs}}} ,
    description = {mean encoder of the \acrshort{vae}, \ie, $\expectation{\gls{latent}\sim\gls{q}}\parenthesis{\gls{latent}}$, mapping $\gls{obs}\mapsto\gls{latent}$} ,
    type        = abbrev,
    parent      = vaes,
}

\newglossaryentry{var_cov}{
    name        = \ensuremath{\gls{cov}^{\gls{encpar}}_{\gls{latent}|\gls{obs}}} ,
    description = {covariance matrix of \gls{q}} ,
    type        = abbrev,
    parent      = vaes,
}

\newglossaryentry{sigmak}{
    name        = \ensuremath{{\sigma}_{k}^{\gls{encpar}}(\gls{obs})^{2}} ,
    description = {variance of \gls{q} in dimension $k$} ,
    type        = abbrev,
    parent      = vaes,
}

\newglossaryentry{sigmaopt}{
    name        = \ensuremath{\boldsymbol{\sigma}^{\gls{encparopt}}(\gls{obs})^{2}} ,
    description = {optimal variance of \gls{q}} ,
    type        = abbrev,
    parent      = vaes,
}

\newglossaryentry{sigmaoptk}{
    name        = \ensuremath{{\sigma}_{k}^{\gls{encparopt}}(\gls{obs})^{2}} ,
    description = {optimal variance of \gls{q} in dimension $k$} ,
    type        = abbrev,
    parent      = vaes,
}

\newglossaryentry{mu}{
    name        = \ensuremath{\boldsymbol{\mu}^{\gls{encpar}}(\gls{obs})} ,
    description = {mean of \gls{q}} ,
    type        = abbrev,
    parent      = vaes,
}

\newglossaryentry{muk}{
    name        = \ensuremath{{\mu}_{k}^{\gls{encpar}}(\gls{obs})} ,
    description = {mean of \gls{q} in dimension $k$} ,
    type        = abbrev,
    parent      = vaes,
}

\newglossaryentry{muopt}{
    name        = \ensuremath{\boldsymbol{\mu}^{\gls{encparopt}}(\gls{obs})} ,
    description = {optimal mean of \gls{q}} ,
    type        = abbrev,
    parent      = vaes,
}

\newglossaryentry{muoptk}{
    name        = \ensuremath{{\mu}_{k}^{\gls{encparopt}}(\gls{obs})} ,
    description = {optimal mean of \gls{q} in dimension $k$} ,
    type        = abbrev,
    parent      = vaes,
}

\newglossaryentry{gamma}{
    name        = \ensuremath{\gamma} ,
    description = {square root of the precision of the \gls{vae} decoder} ,
    type        = abbrev,
    parent      = vaes,
}

\newglossaryentry{betaloss}{
    name        = \ensuremath{\mathcal{L}_{\beta}} ,
    description = {\betavae loss function} ,
    type        = abbrev,
    parent      = vaes,
}

\newglossaryentry{pxz}{
    name        = \ensuremath{p_{\gls{decpar}}(\gls{obs}|\gls{latent})} ,
    description = {conditional distribution of the decoded samples of the \acrshort{vae}, mapping $\gls{latent}\mapsto\gls{obs}$, parametrized by \gls{decpar}} ,
    type        = abbrev,
    parent      = vaes,
}
\newglossaryentry{pzx}{
    name        = \ensuremath{p_{\gls{decpar}}(\gls{latent}|\gls{obs})} ,
    description = {true posterior distribution of the decoded samples of the \acrshort{vae}, mapping $\gls{obs}\mapsto\gls{latent}$, parametrized by \gls{decpar}} ,
    type        = abbrev,
    parent      = vaes,
}
\newglossaryentry{decpar}{
    name        = \ensuremath{\boldsymbol{\theta}} ,
    description = {parameters of the decoder \gls{pxz}} ,
    type        = abbrev,
    parent      = vaes,
}

\newglossaryentry{invdeccomp}{
    name        = \ensuremath{{g}^{\gls{decpar}}} ,
    description = {inverse decoder component} ,
    type        = abbrev,
    parent      = vaes,
}

\newglossaryentry{invdec}{
    name        = \ensuremath{\mathrm{\boldsymbol{g}}^{\gls{decpar}}} ,
    description = {inverse decoder} ,
    type        = abbrev,
    parent      = vaes,
}

\newglossaryentry{distortion}{
    name        = \ensuremath{D} ,
    description = {Distortion of \cite{alemi_fixing_2018}, the same as the reconstruction term of the \acrshort{elbo} for $\beta=1$} ,
    type        = abbrev,
    parent      = vaes,
}

\newglossaryentry{rate}{
    name        = \ensuremath{R} ,
    description = {Rate of \cite{alemi_fixing_2018}, the same as the \acrshort{kld} term of the \acrshort{elbo} for $\beta=1$} ,
    type        = abbrev,
    parent      = vaes,
}

\newglossaryentry{lindec}{
    name        = \ensuremath{\boldsymbol{\mathrm{W}}} ,
    description = {weight matrix of a linear decoder} ,
    type        = abbrev,
    parent      = vaes,
}

\newglossaryentry{linenc}{
    name        = \ensuremath{\boldsymbol{\mathrm{V}}} ,
    description = {weight matrix of a linear encoder} ,
    type        = abbrev,
    parent      = vaes,
}

\newglossaryentry{imas}{type=abbrev,name=\acrlong{ima},description={\nopostdesc}}

\newglossaryentry{mixing}{
    name        = \ensuremath{\inv{g}} ,
    description = {inverse of the learned unmixing of the \acrshort{ima}, mapping $\gls{latent}\mapsto\gls{obs}$ } ,
    type        = abbrev,
    parent      = imas,
}

\newglossaryentry{lin_mixing}{
    name        = \ensuremath{A} ,
    description = {ground-truth \emph{linear} mixing process of the \acrshort{ima}, mapping $\gls{latent}\mapsto\gls{obs}$ } ,
    type        = abbrev,
    parent      = imas,
}

\newglossaryentry{cima_local}{
    name        = \ensuremath{c_{\acrshort{ima}}} ,
    description = {local \acrshort{ima} contrast } ,
    type        = abbrev,
    parent      = imas,
}

\newglossaryentry{cima_global}{
    name        = \ensuremath{C_{\acrshort{ima}}} ,
    description = {global \acrshort{ima} contrast } ,
    type        = abbrev,
    parent      = imas,
}

\newglossaryentry{source}{
    name        = \ensuremath{s} ,
    description = {sources (\acrshort{ica} equivalent of latents)} ,
    type        = abbrev,
    parent      = imas,
}

\newglossaryentry{rec_s}{
    name        = \ensuremath{\boldsymbol{y}} ,
    description = {reconstructed sources} ,
    type        = abbrev,
    parent      = imas,
}

\newglossaryentry{p_source}{
    name        = \ensuremath{p_{\gls{latent}}} ,
    description = {source distribution} ,
    type        = abbrev,
    parent      = imas,
}

\newglossaryentry{imaloss}{
    name        = \ensuremath{\mathcal{L}_{\gls{ima}}} ,
    description = {\gls{ima} loss function} ,
    type        = abbrev,
    parent      = imas,
}

\NewDocumentCommand{\cima}{ O{\gls{dec}} O{\gls{latent}}  }{\ensuremath{\gls{cima_local} ( #1\!,  #2) }\xspace}
\NewDocumentCommand{\Cima}{ O{\gls{dec}} O{\ensuremath{p_0}  }}{\ensuremath{\gls{cima_global} ( #1,  #2) }\xspace}

\newglossaryentry{gps}{type=abbrev,name=\acrlongpl{gp},description={\nopostdesc}}
\newglossaryentry{gpr}{
    name        = \ensuremath{\mathcal{GP}} ,
    description = {Gaussian Process} ,
    type        = abbrev,
    parent      = gps,
}

\newglossaryentry{gpker}{
    name        = \ensuremath{k} ,
    description = {kernel function} ,
    type        = abbrev,
    parent      = gps,
}

\newglossaryentry{gpcov}{
    name        = \ensuremath{\mathcal{K}} ,
    description = {$\gls{numdata}\times\gls{numdata}$ covariance matrix of a \acrshort{gp}} ,
    type        = abbrev,
    parent      = gps,
}

\makeglossaries

\newacronym{mpa}{MPA}{Measure Preserving Automorphism}
\newacronym{iid}{i.i.d.}{independent and identically distributed}
\newacronym{vmf}{vMF}{von Mises-Fisher}
\newacronym{nivmf}{nivMF}{non-isotropic von Mises-Fisher}
\newacronym{pd}{PD}{positive definite}
\newacronym{psd}{PSD}{positive semi-definite}
\newacronym{nd}{ND}{negative definite}
\newacronym{nsd}{NSD}{negative semi-definite}

\newacronym{ode}{ODE}{Ordinary Differential Equation}
\newacronym{pde}{PDE}{Partial Differential Equation}

\newacronym{lhs}{LHS}{left hand side}
\newacronym{rhs}{RHS}{right hand side}

\newacronym{rv}{RV}{random variable}

\newacronym{ae}{AE}{AutoEncoder}
\newacronym{lae}{LAE}{Linear Autoencoder}
\newacronym{vae}{VAE}{Variational Autoencoder}
\newacronym{cvvae}{CV-VAE}{Constant-Variance Variational Autoencoder}
\newacronym{ivae}{iVAE}{Identifiable Variational Autoencoder}
\newacronym{rae}{RAE}{Regularized Autoencoder}
\newacronym{grae}{GRAE}{Gaussian Regularized Autoencoder}
\newacronym{lvm}{LVM}{latent variable model}

\newacronym[longplural=Gaussian Processes]{gp}{GP}{Gaussian Process}
\newacronym{gplvm}{GPLVM}{Gaussian Process Latent Variable Model}
\newacronym{rbf}{RBF}{Radial Basis Function}

\newcommand{\betavae}{$\beta$-\gls{vae}\xspace}

\newacronym{kld}{KL}{Kullback-Leibler Divergence}
\newacronym{elbo}{{\text{\upshape ELBO}}}{evidence lower bound}
\newacronym{pca}{PCA}{Principal Component Analysis}
\newacronym{ppca}{PPCA}{Probabilistic Principal Component Analysis}
\newacronym{ebm}{EBM}{Energy-Based Model}
\newacronym{cca}{CCA}{Canonical Correlation Analysis}

\newacronym{mi}{MI}{Mutual Information}

\newacronym[longplural=Identifiable Exchangeable Mechanisms]{iem}{IEM}{Identifiable Exchangeable Mechanisms}
\newacronym[longplural=Independent Causal Mechanisms]{icm}{ICM}{Independent Causal Mechanisms}
\newacronym{sms}{SMS}{Sparse Mechanism Shift}
\newacronym{mss}{MSS}{Mechanism Shift Score}
\newacronym{sem}{SEM}{Structural Equation Model}
\newacronym{lingam}{LiNGAM}{Linear Non-Gaussian Acyclic Model}
\newacronym{dag}{DAG}{Directed Acyclic Graph}
\newacronym{anm}{ANM}{Additive Noise Model}
\newacronym{cd}{CD}{Causal Discovery}
\newacronym{crl}{CRL}{Causal Representation Learning}
\newacronym{hmm}{HMM}{Hidden Markov Model}
\newacronym{plr}{PLR}{Partially Linear Regression}

\newacronym{it}{IT}{identifiability theory}
\newacronym{sith}{SITh}{Singular Identifiability Theory}

\newacronym{lt}{LT}{learning theory}
\newacronym{slt}{SLT}{Singular Learning Theory}

\newacronym{ica}{ICA}{Independent Component Analysis}
\newacronym{nlica}{NLICA}{nonlinear Independent Component Analysis}
\newacronym{bss}{BSS}{Blind Source Separation}

\newacronym{ima}{{\text{\upshape IMA}}}{Independent Mechanism Analysis}
\newacronym{igci}{IGCI}{Information Geometric Causal Inference}
\newacronym{cdf}{CdF}{Causal de Finetti}

\newacronym{nce}{NCE}{Noise Contrastive Estimation}
\newacronym{pcl}{PCL}{Permutation-Contrastive Learning}
\newacronym{tcl}{TCL}{Time-Contrastive Learning}
\newacronym{gencl}{GCL}{Generalized Contrastive Learning}
\newacronym{iia}{IIA}{Independent Innovation Analysis}

\newacronym{ar}{AR}{autoregressive}
\newacronym{var}{VAR}{Vector autoregressive}
\newacronym{nvar}{NVAR}{Nonlinear Vector AutoRegressive}

\newacronym{ai}{AI}{Artificial Intelligence}
\newacronym{ml}{ML}{Machine Learning}
\newacronym{dml}{DML}{Double Machine Learning}
\newacronym{oml}{OML}{Orthogonal Machine Learning}
\newacronym{homl}{HOML}{Higher-order Orthogonal Machine Learning}
\newacronym{dl}{DL}{Deep Learning}
\newacronym{rl}{RL}{Reinforcement Learning}
\newacronym{mbrl}{MBRL}{Model-Based Reinforcement Learning}
\newacronym{rlhf}{RLHF}{Reinforcement Learning from Human Feedback}
\newacronym{ssl}{SSL}{self-supervised learning}

\newacronym{cl}{CL}{Contrastive Learning}
\newacronym{dcl}{DCL}{Debiased Contrastive Learning}
\newacronym{scl}{SCL}{Spectral Contrastive Learning}
\newacronym{gcl}{GCL}{Graph Contrastive Learning}
\newacronym{alphacl}{$\alpha$-CL}{$\alpha$-Contrastive Learning}
\newacronym{arcl}{ArCL}{Augmentation-robust Contrastive Learning}
\newacronym{fce}{FCE}{Flow Contrastive Estimation}
\newacronym{pid}{PID}{parametric instance discrimination}

\newacronym{vince}{VINCE}{Variational InfoNCE}
\newacronym{rince}{RINCE}{Robust InfoNCE}
\newacronym{aggnce}{AggNCE}{Aggregated InfoNCE}
\newacronym{mcinfonce}{MCInfoNCE}{Monte-Carlo InfoNCE}

\newacronym{gmc}{GMC}{Geometric Multimodal Contrastive Learning}
\newacronym{looc}{LooC}{Leave-one-out Contrastive Learning}
\newacronym{npc}{NPC}{Negative-Positive Coupling}
\newacronym{cpc}{CPC}{Contrastive Predictive Coding}
\newacronym{nlp}{NLP}{Natural Language Processing}
\newacronym{gdl}{GDL}{Geometric Deep Learning}
\newacronym{msn}{MSN}{Masked Siamese Networks}
\newacronym{ifm}{IFM}{Implicit Feature Modification}

\newacronym{dnn}{DNN}{Deep Neural Network}
\newacronym{nn}{NN}{Neural Network}
\newacronym{ann}{ANN}{Artificial Neural Network}

\newacronym{fm}{FM}{Foundation Model}
\newacronym{llm}{LLM}{Large Language Model}
\newacronym{vlm}{VLM}{Vision-Language Model}
\newacronym{pcfg}{PCFG}{Probabilistic Context-Free Grammar}
\newacronym{icl}{ICL}{in-context learning}

\newacronym{nc}{NC}{Neural Collapse}
\newacronym{cdt}{CDT}{Class-Dependent Temperature}
\newacronym{mlp}{MLP}{Multi-Layer Perceptron}
\newacronym{fc}{FC}{Fully Connected}
\newacronym{strnn}{StrNN}{Structured Neural Network}

\newacronym{cn}{conv}{Convolutional layer}
\newacronym{cnn}{CNN}{Convolutional Neural Network}
\newacronym{gnn}{GNN}{Graph Neural Network}
\newacronym{ssm}{SSM}{State Space Model}

\newacronym{rnn}{RNN}{Recurrent Neural Network}
\newacronym{lstm}{LSTM}{Long Short-Term Memory}
\newacronym{gru}{GRU}{Gated Recurrent Unit}
\newacronym{relu}{ReLU}{Rectified Linear Unit}
\newacronym{bn}{BN}{Batch Normalization}
\newacronym{dbn}{DBN}{Decorrelated Batch Normalization}

\newacronym{gan}{GAN}{Generative Adversarial Network}

\newacronym{diayn}{DIAYN}{Diversity Is All You Need}
\newacronym{dads}{DADS}{DYnamics-Aware Discovery of Skills}
\newacronym{sac}{SAC}{Soft Actor Critic}
\newacronym{a2c}{A2C}{Advantage Actor Critic}

\newacronym{sgd}{SGD}{Stochastic Gradient Descent}
\newacronym{adam}{ADAM}{Adaptive Moment Estimation}
\newacronym{svd}{SVD}{Singular Value Decomposition}
\newacronym{wls}{WLS}{Weighted Least Squares}

\newacronym{sam}{SAM}{Sharpness-Aware Minimization}
\newacronym{samba}{SAMBA}{SAM-Based Autoencoder}
\newacronym{vi}{VI}{Variational Inference}
\newacronym{mfvi}{MFVI}{Mean Field Variational Inference}

\newacronym[longplural=data generating processes]{dgp}{DGP}{data generating process}
\newacronym{map}{MAP}{Maximum A Posteriori}
\newacronym{mle}{MLE}{maximum likelihood estimation}

\newacronym{etf}{ETF}{Equiangular Tight Frame}

\newacronym{mse}{MSE}{Mean Squared Error}
\newacronym{mae}{MAE}{Mean Absolute Error}
\newacronym{ce}{{\text{\upshape CE}}}{cross entropy}

\newacronym{sid}{SID}{Structural Intervention Distance}
\newacronym{shd}{SHD}{Structural Hamming Distance}

\newacronym{mcc}{MCC}{Mean Correlation Coefficient}
\newacronym{mig}{MIG}{Mutual Information Gap}
\newacronym{dci}{DCI}{Disentanglement Completeness Informativeness score}

\newacronym{arc}{ARC}{Average Relative Confusion}
\newacronym{acr}{ACR}{Average Confusion Ratio}

\newacronym{api}{API}{Application Programming Interface}
\newacronym{cpu}{CPU}{Central Processing Unit}
\newacronym{gpu}{GPU}{Graphics Processing Unit}

\newacronym{lti}{LTI}{Linear Time-Invariant}
\newacronym{zoh}{ZOH}{Zero-Order Hold}

\newacronym{gt}{{\text{\upshape GT}}}{ground truth}
\newacronym{ood}{OOD}{out-of-distribution}
\newacronym{oov}{OOV}{out-of-variablme}
\newacronym{fsm}{FSM}{Finite State Machine}
\newacronym{rasp}{RASP}{Restricted-Access Sequence Processing Language}

\newacronym{ntk}{NTK}{Neural Tangent Kernel}

\newacronym{as}{a.s.}{almost surely}
\newacronym{alev}{a.e.}{almost everywhere}

\newacronym{sos}{SOS}{start-of-sequence}
\newacronym{eos}{EOS}{end-of-sequence}

\newacronym{prh}{PRH}{Platonic Representation Hypothesis}
\usepackage{pifont}
\usepackage{multirow}

\definecolor{figblue}{HTML}{4A90E2}
\definecolor{figred}{HTML}{D0021B}
\definecolor{figgreen}{HTML}{2CA02C}

\title{Estimating Treatment Effects with Independent Component Analysis}

\author[1,2]{Patrik~Reizinger\thanks{This work was initiated during P.R.'s internship at the Vector Institute. Correspondence to: \href{mailto:patrik.reizinger@tuebingen.mpg.de}{\texttt{patrik.reizinger@tuebingen.mpg.de}}}}
\author[3]{Lester~Mackey}
\author[1]{Wieland~Brendel}
\author[2,4]{Rahul~G.~Krishnan}

\affil[1]{%
    Max Planck Institute for Intelligent Systems \& ELLIS Institute\\
    T\"ubingen, Germany
}
\affil[2]{%
    Vector Institute\\
    Toronto, Canada
}
\affil[3]{%
    Microsoft Research\\
    New England, USA
}
\affil[4]{%
    Department of Computer Science, University of Toronto\\
    Toronto, Canada
}

\begin{document}
\etocdepthtag.toc{main}
\maketitle

\begin{abstract}

Independent Component Analysis (ICA) uses a measure of non-Gaussianity to identify latent sources from data and estimate their mixing coefficients \citep{shimizu_linear_2006}. Meanwhile, higher-order Orthogonal Machine Learning (OML) exploits non-Gaussian treatment noise to provide more accurate estimates of treatment effects in the presence of confounding nuisance effects \citep{mackey2018orthogonalICML}. Remarkably, we find that the two approaches rely on the same moment conditions for consistent estimation. We then seize upon this connection to show how ICA can be effectively used for treatment effect estimation. Specifically, we prove that linear ICA can consistently estimate multiple treatment effects, even in the presence of Gaussian confounders, and identify regimes in which ICA is provably more sample-efficient than OML for treatment effect estimation. Our synthetic demand estimation experiments confirm this theory and demonstrate that linear ICA can accurately estimate treatment effects even in the presence of nonlinear nuisance.
\end{abstract}

\section{Introduction}

This work initiates the study of \gls{ica}~\citep{comon1994independent,hyvarinen_independent_2000} for treatment effect estimation in the \gls{plr} model~\citep{robinson1988root}. 
The accurate estimation of causal effects is a central challenge in medical research and policy-making~\citep{king1994designing}, as it guides the development of more effective treatment strategies and interventions~\citep{rosenbaum1983central,pearl2009causal,hill2011bayesian}. This task becomes difficult when the data contain high-dimensional confounding variables---features that affect both the treatment and the outcome. A number of machine learning methods have been developed to handle this setting while maintaining theoretical guarantees on treatment effect estimation. %
Among these methods, 
\gls{oml}~\citep{chernozhukov_doubledebiased_2017,mackey2018orthogonalICML,jin2025its} exhibits robust statistical properties in the \gls{plr} model ~\citep{robinson1988root}, where confounders affect the outcome and treatment in a potentially nonlinear way. OML’s two-stage procedure---first learning nuisance functions, then leveraging orthogonalization to adjust for confounders---yields consistent and efficient estimators of treatment effects under minimal assumptions. %

    \gls{ica} is a family of representation learning methods for separating mixed signals into statistically independent components, enabling the discovery of latent causal representations  from observational data. While \gls{ica} is widely used for causal inference tasks~\citep{tramontano_causal_2024,khemakhem_causal_2021,wendong_causal_2023}, including linear~\citep{shimizu_linear_2006} and  nonlinear~\citep{reizinger_jacobian-based_2023} \gls{cd}, \ie, the extraction of a causal graph, the potential of ICA for treatment effect estimation is still underdeveloped. %
    For example, the empirical studies of \citet{ribeiro23highfidelity,jiang2023causal} employ nonlinear ICA for effect estimation but do not establish its consistency or analyze its estimation quality. Meanwhile, \citet{shimizu_linear_2006} describe sufficient conditions for recovering the mixing matrix in linear ICA, but their assumptions are stronger than necessary, and they do not explore the important application of treatment effect estimation.

            \begin{figure*}
    		\centering
                \tikzset{every picture/.style={line width=0.75pt}} %

\begin{tikzpicture}[x=0.75pt,y=0.75pt,yscale=-.65,xscale=1]
\draw [color={rgb, 255:red, 74; green, 144; blue, 226 }  ,draw opacity=1 ]   (51,83.95) -- (85,84) ;
\draw [shift={(88,84)}, rotate = 180.08] [fill={rgb, 255:red, 74; green, 144; blue, 226 }  ,fill opacity=1 ][line width=0.08]  [draw opacity=0] (8.93,-4.29) -- (0,0) -- (8.93,4.29) -- cycle    ;
\draw [color={rgb, 255:red, 208; green, 2; blue, 27 }  ,draw opacity=1 ]   (62,47.75) -- (45.14,64.39) ;
\draw [shift={(43,66.5)}, rotate = 315.38] [fill={rgb, 255:red, 208; green, 2; blue, 27 }  ,fill opacity=1 ][line width=0.08]  [draw opacity=0] (8.93,-4.29) -- (0,0) -- (8.93,4.29) -- cycle    ;

\draw (25,111.4) node [anchor=north west][inner sep=0.75pt]    {$ \begin{array}{l}
X\ =\ \xi \\
T\ = \ \textcolor[rgb]{0.82,0.01,0.11}{a} X\ +\ \eta \\
Y\ =\ \textcolor[rgb]{0.17,0.63,0.17}{b} X\ +\textcolor[rgb]{0.29,0.56,0.89}{\theta } T\ +\ \varepsilon 
\end{array}$};
\draw (98.4,75.6) node [anchor=north west][inner sep=0.75pt] {$Y$};
\draw (34.08,76) node [anchor=north west][inner sep=0.75pt] {$T$};
\draw (63.08,33) node [anchor=north west][inner sep=0.75pt] {$X$};
\draw (4,7) node [anchor=north west][inner sep=0.75pt] {\footnotesize \textbf{Partially Linear Regression} };

\draw    (453,82) -- (601,82) ;
\draw [shift={(604,82)}, rotate = 180] [fill={rgb, 255:red, 0; green, 0; blue, 0 }  ][line width=0.08]  [draw opacity=0] (8.93,-4.29) -- (0,0) -- (8.93,4.29) -- cycle    ;

\draw (409,51.4) node [anchor=north west][inner sep=0.75pt] {$\begin{bmatrix}
X\\
T\\
Y
\end{bmatrix}$};
\draw (467,93.4) node [anchor=north west][inner sep=0.75pt] {
$
\underbrace{
                    \begin{bmatrix}
                        1 & 0 & 0\\
                        {\textcolor[rgb]{0.82,0.01,0.11}{-a}} & 1 & 0\\
                        {\textcolor[rgb]{0.17,0.63,0.17}{-b}} & \textcolor[rgb]{0.29,0.56,0.89}{-\theta } & 1
                    \end{bmatrix}
                    }_{=\mat{W}\ \text{(\textit{un}mixing matrix})}
$
};
\draw (609,52.4) node [anchor=north west][inner sep=0.75pt] {$\begin{bmatrix}
{\xi }\\
{\eta }\\
{\varepsilon }
\end{bmatrix}$};
\draw (398,7) node [anchor=north west][inner sep=0.75pt] {\footnotesize\textbf{Independent Component Analysis}};

\draw (180,10) -- (180,169);
\draw (396,10) -- (396,169);

\draw (192,7) node [anchor=north west][inner sep=0.75pt] {\footnotesize\textbf{Orthogonal Machine Learning}};

\draw (238.08,46) node {$X$};
\draw (200.08,75) node {$T$};
\draw [color={rgb,255:red,208; green,2; blue,27}] (232,48) -- (208,74);
\draw [shift={(208,72)}, rotate=322, fill={rgb,255:red,208; green,2; blue,27}] (8.93,-4.29)--(0,0)--(8.93,4.29)--cycle;
\draw   (198.32,95.28) .. controls (188.18,88.52) and (189.42,68.82) .. (201.11,51.28) .. controls (212.79,33.74) and (230.49,25) .. (240.63,31.76) .. controls (250.78,38.52) and (249.53,58.22) .. (237.84,75.76) .. controls (226.16,93.3) and (208.47,102.04) .. (198.32,95.28) -- cycle ;
\draw (216,100) -- (216,130);
\draw [shift={(216,132)}, rotate=270, fill=black] (8.93,-4.29)--(0,0)--(8.93,4.29)--cycle;
\draw (212,147.4) node {$\eta$};
\draw (198,32) node {$1.$};

\draw (298,46) node {$X$};
\draw (260,75) node {$Y$};
\draw [color={rgb,255:red,44; green,160; blue,44}] (292,48) -- (268,74);
\draw [shift={(268,72)}, rotate=322, fill={rgb,255:red,44; green,160; blue,44}] (8.93,-4.29)--(0,0)--(8.93,4.29)--cycle;
\draw   (258.32,95.28) .. controls (248.18,88.52) and (249.42,68.82) .. (261.11,51.28) .. controls (272.79,33.74) and (290.49,25) .. (300.63,31.76) .. controls (310.78,38.52) and (309.53,58.22) .. (297.84,75.76) .. controls (286.16,93.3) and (268.47,102.04) .. (258.32,95.28) -- cycle ;
\draw (270,100) -- (270,130);
\draw [shift={(270,132)}, rotate=270, fill=black] (8.93,-4.29)--(0,0)--(8.93,4.29)--cycle;
\draw (272,147) node {$\varepsilon$};
\draw (262,32) node {$2.$};

\draw (354,46) node {$\eta,\varepsilon$};
\draw (326,75) node {$Y$};
\draw [color={rgb,255:red,74; green,144; blue,226}] (350,50) -- (332,74);
\draw [shift={(332,72)}, rotate=322, fill={rgb,255:red,74; green,144; blue,226}] (8.93,-4.29)--(0,0)--(8.93,4.29)--cycle;
\draw   (318.32,95.28) .. controls (308.18,88.52) and (309.42,68.82) .. (321.11,51.28) .. controls (332.79,33.74) and (350.49,25) .. (360.63,31.76) .. controls (370.78,38.52) and (369.53,58.22) .. (357.84,75.76) .. controls (346.16,93.3) and (328.47,102.04) .. (318.32,95.28) -- cycle ;
\draw (330,100) -- (330,130);
\draw [shift={(330,132)}, rotate=270, fill=black] (8.93,-4.29)--(0,0)--(8.93,4.29)--cycle;
\draw (330,147) node {$\theta$};
\draw (322,32) node {$3.$};

\draw[color={rgb, 255:red, 44; green, 160; blue, 44}, draw opacity=1] (78,48) -- (96.4,73.6);

\draw[shift={(98.4,75.6)}, rotate=229] 
      [fill={rgb, 255:red, 44; green, 160; blue, 44}, fill opacity=1]
      [line width=0.08] [draw opacity=0]
      (8.93,-4.29) -- (0,0) -- (8.93,4.29) -- cycle;

\end{tikzpicture}
                \caption{\textbf{Overview of treatment effect estimation in the \acrfull{plr} model.} \textbf{(Left:)} The linear \gls{plr} model, where the covariates $X$ affect both treatment $T$ and outcome $Y$. The quantity of interest is the treatment effect {\color{figblue}$\theta$}. \textbf{(Center:)} \acrfull{oml} estimates  {\color{figblue}$\theta$} in three steps.  
                \textbf{(Right:)} \acrfull{ica} can invert the \gls{plr} model by maximizing non-Gaussianity of the sources, thereby yielding {\color{figblue}$\theta$} as a coefficient in the \emph{unmixing matrix} $\mat{W}$. Scale and permutation indeterminacies are resolved by relying on non-Gaussianity and the \gls{plr} structure (\cref{lem:lin_plr_ica}).}
                \vspace{-1.25\baselineskip}
                \label{fig:fig1}
    	\end{figure*}

    Interestingly, non-Gaussianity plays a crucial role in both \gls{ica} and treatment effect estimation. %
    For (linear) \gls{ica}, non-Gaussianity is required to unambiguously disentangle the component sources, while, for treatment effect estimation, non-Gaussianity is required to achieve higher-order robustness in the PLR model~\citep{mackey2018orthogonalICML,jin2025its}. %
    We seize upon this link and build upon the results of~\citet{shimizu_linear_2006,mackey2018orthogonalICML} to show that \gls{ica} can be used to effectively estimate treatment effects, both in theory and in practice.
    Focusing on the \gls{plr} model, we first prove that \gls{ica} can estimate treatment effects by reducing the problem to identifying the elements of the  \gls{ica} unmixing matrix (\cf \cref{fig:fig1}). Next, we show how the permutation and scale indeterminacies of \gls{ica} can be overcome by leveraging non-Gaussianity and the \gls{plr} structure. Remarkably, this construction enables effect estimation for multiple treatments and accommodates arbitrary (even Gaussian) covariate noise, all using the same off-the-shelf \gls{ica} algorithm, FastICA~\citep{hyvarinen1999fast}. Experimentally, and perhaps surprisingly, we also demonstrate how to use \textit{linear} \gls{ica} for estimating treatment effects in a nonlinear \gls{plr}. 
    Our \textbf{primary contributions} can be summarized as follows:
    \begin{itemize}[leftmargin=*,nolistsep]
        \item We formalize the link between \acrfull{ica} and \acrfull{oml} for \gls{plr} treatment effect estimation, clarifying the role of non-Gaussianity in both algorithms (\cref{sec:non-gaussian}) and identifying regimes in which each method is more sample efficient than the other (\cref{subsec:asymp_var}).
        \item We prove that ICA can be used to consistently estimate multiple treatment effects (\cref{lem:lin_plr_ica,cor:multi_T}) and even effects with partially Gaussian source variables (\cref{table:breaking_symmetries,cor:ica_gauss}).
        \item We complement our theory with a diversity of experiments simulating demand estimation from purchasing and pricing data; testing the effect of model coefficients on asymptotic variance;  examining estimator performance under varying sample sizes, covariate dimensions, treatment dimensions, nonlinearities, and covariate noise distributions; and benchmarking the quality of linear \gls{ica} for non-linear treatment effect estimation (\cref{sec:exp}). 
        In each case, \gls{ica} compares favorably with state-of-the-art \gls{oml} approaches specifically developed for the \gls{plr} setting.
    \end{itemize}

\section{Background}

        \paragraph{\acrfull{plr}.} 
        The \gls{plr} model of~\citet{robinson1988root} models the relationship between a scalar treatment $T$, a scalar outcome $Y$, and potentially confounding covariates $X$:
        \vspace{-.25em}
        \begin{align}\label{eq:plr}
            T = g(X) +\eta \sstext{and} Y = \theta T  + f(X) + \varepsilon\\
            \sstext{for independent}
            (\eps, \eta) \sstext{with} \E[\eps] = \E[\eta] = 0.\notag
            \vspace{-.4em}
        \end{align}
        Above, the coefficient $\theta$ is termed the \emph{treatment effect}, $f$ and $g$ are nuisance functions capturing the influence of $X$, and $\eps$ and $\eta$ are random noise variables.

         \paragraph{Causality.}
            Causality models cause and effect relationships as a \gls{dag} between variables with functional relationships often specified by \glspl{sem}~\citep{pearl_causality_2009,peters_elements_2018,spirtes_causal_2016}. A \gls{sem} consists of independent exogenous noise variables $S$ (causes), dependent endogenous causal variables $Z$ (effects), and functional mechanisms $h$ describing the relationship between the variables, \ie, $Z_i \defeq h_i(\parent(Z_i), S_i),$ where $\parent(Z_i)$ denotes the parent variables of $Z_i$ in the graph.
            A special \gls{sem} family is that of \glspl{anm}, where the exogenous variable $S_i$ affects $Z_i$ additively, \ie, $Z_i \defeq h_i(\parent(Z_i)) + S_i.$ Notably, as we illustrate in \cref{fig:fig1}, the \gls{plr} model can be framed as an \gls{anm} with $Z=(X,Y,T)$ and $S=(\xi,\eta,\eps)$ for $\xi\defeq X$. We denote the covariate dimension by $d\defeq \dim X$. 
            We make the standard assumptions of no unobserved confounding and positivity, i.e., $\;\forall\;t : P(T = t \mid X) > 0 $.
        
     \paragraph{\acrfull{ica}.}
            \gls{ica} models the observations
            as a deterministic mixture of \textit{independent} sources~\citep{comon1994independent, hyvarinen_independent_2001}.
             The estimation goal for ICA is the recovery of the latent factors, and ICA provides identifiability guarantees for the factors in the infinite sample limit. Identifiability means that the ground-truth latent factors are recovered up to simple indeterminacies such as permutation and element-wise transformations. That is, for source variables $S$ and observed mixtures $Z=\mat{A}S$, the objective of \gls{ica} is to recover $S = \mat{W}Z$, where $\mat{A}$ is an unobserved \emph{mixing matrix} and $\mat{W} \defeq \mat{A}^{-1}$ is the \emph{unmixing matrix}. 
             Identifiability necessitates certain assumptions, even in the linear case, a central one being the non-Gaussianity of $S$. 
            This can be seen, for example, by estimating the sources via maximizing the data log-likelihood, 
            which is expressed in terms of the sources by the change-of-variables formula:
             \vspace{-.6em}
             \begin{align*}
                  \log p_Z(Z) &= \log p_S(S) + \logabsdet{\mat{W}},
             \end{align*}
             If more than one of the components of $S$ is Gaussian, then rotating those sources does not change the likelihood. This is due to the rotation invariance of a Gaussian $p_S(S)$ and the fact that any orthogonal matrix \mat{O} preserves the absolute determinant, \ie, $\abs{\det( \mat{W}\mat{O})}=\abs{\det \mat{W}\det{\mat{O}}}=\abs{\det\mat{W}}$ since $\abs{\det\mat{O}}=1$.
             Indeed, the goal of ICA is to find the most non-Gaussian directions in the data by maximizing a measure of non-Gaussianity.  Common choices include maximizing the likelihood, maximizing kurtosis as in the original  FastICA algorithm~\citep{hyvarinen1997fast} (\cf \cref{alg:fastica}), or maximizing negative entropy~\citep{hyvarinen1997new} or its approximation \texttt{logcosh} to minimize mutual information~\citep{hyvarinen1999fast}. 
             In each case, a different contrast function $U$ is employed as a measure of non-Gaussianity (\cf \cref{subsec:asymp_var} for details). 

            Nonlinear \gls{ica} is usually impossible without further assumptions~\citep{darmois1951analyse, hyvarinen_nonlinear_1999,locatello_challenging_2019}. Recent developments have relaxed the independence condition to conditional independence and proved identifiability in the nonlinear case~\citep{hyvarinen_nonlinear_2019, gresele_incomplete_2019, khemakhem_variational_2020, halva_disentangling_2021,hyvarinen_unsupervised_2016,khemakhem_ice-beem_2020,locatello_weakly-supervised_2020,morioka_connectivity-contrastive_2023,morioka_independent_2021,reizinger_identifiable_2024,reizinger_cross-entropy_2024}. 
            These methods often rely on data from multiple environments and require these environments to be ``sufficiently diverse''. Examples include non-stationary time series or patient data collected at different hospitals with different socioeconomic and health statuses.
            The connection between \gls{ica} and \acrfull{cd} is well-known in the linear case of LiNGAM~\citep{shimizu_linear_2006}, and it was recently shown in the nonlinear case by~\citet{reizinger_jacobian-based_2023}.
            To the best of our knowledge, \gls{ica} has not been theoretically studied for treatment effect estimation, and this is the focus of the present work.
            Large- and finite-sample behavior of ICA estimators are also generally not the focus of identifiability research, though there exist several relevant results for the linear case~\citep{yin_statistical_2007,bermejo_finite_2007,auddy_large_2023}.

        \paragraph{Treatment Effect Estimation.} 
        Treatment effect estimation in the \gls{plr} model \cref{eq:plr} focuses on the estimation of the coefficient $\theta$. %
        A wide variety of statistical methods have been developed for estimating treatment effects including \acrfull{oml}~\citep{chernozhukov_doubledebiased_2017,mackey2018orthogonalICML,jin2025its}, targeted maximum likelihood estimation~\citep{schuler_targeted_2017}, propensity score-based techniques such as inverse probability of treatment weighting~\citep{feng_generalized_2012,mccaffrey_tutorial_2013}, and Bayesian additive regression trees~\citep{chipman_bart_2010}, as well as extensions for multiple treatments~\citep{hu_estimation_2020,xiang2025double}, high-dimensional and sparse treatments~\citep{zhu_high_2019}, multi-level treatments~\citep{Shi06012025}, multiple heterogeneous environments~\citep{kivva_causal_2025}, and shared states~\citep{hays_double_2025}. %

        Recently, \citet{jin2025its} showed that first-order \gls{oml} \citep{chernozhukov_doubledebiased_2017} provides minimax rate-optimal treatment effect estimates when the \gls{plr} treatment noise is Gaussian and, surprisingly, that higher-order \gls{oml} \citep{mackey2018orthogonalICML} yields even higher quality estimates whenever the treatment noise is \textbf{not} Gaussian. %
        Hence, while treatment effect estimation is still possible with Gaussian treatment noise, it is subject to a Gaussian quality barrier, reminiscent of the Gaussian identifiability barrier of \gls{ica}.
        We will next dig deeper into this connection and uncover a path toward \gls{ica}-based treatment effect estimation. 
\subsection{Role of Non-Gaussianity in ICA and OML}\label{sec:non-gaussian}

\begin{table}[t]
  \centering
  \scriptsize %
  \setlength{\tabcolsep}{2pt} %
  \renewcommand{\arraystretch}{0.95} %
  \begin{adjustbox}{max width=\columnwidth}
  \begin{tabular}{@{}l l c c@{}} \toprule\midrule
    \textbf{Method} & \textbf{Reference} & \textbf{Noise} & \textbf{Output}\\\midrule\midrule
    \multirow{2}{*}{\gls{oml}} & \citet{chernozhukov_doubledebiased_2017} & any & $\hat\theta$\\
     &  \citet{mackey2018orthogonalICML} & non-G $T$  & $\hat\theta^{\sqrt{n}}$ \\\midrule
    \multirow{2}{*}{ICA}  & \cref{lem:lin_plr_ica} &  non-G & $\hat\theta, \hat S$ \\
    & \cref{cor:ica_gauss} &  non-G $T$\&$Y$ & $\hat\theta$ \\
    \midrule
    \bottomrule
  \end{tabular}
  \end{adjustbox}
  \caption{\textbf{Assumptions for treatment effect estimation and source recovery under the \gls{plr} (equivalently, \gls{anm}) model.} G is shorthand for Gaussian, ${\theta}$ for the treatment effect with $\sqrt{n}$ denoting improved estimation consistency, and $S$ for the noise variables. $\hat{\cdot}$ indicates \textit{estimated} quantities. ICA only uses the PLR structure for selecting the treatment effect coefficient from the unmixing matrix but not for source estimation.}
  \label{table:breaking_symmetries}
\end{table}

        Linear ICA is impossible with more than one Gaussian source, as the rotational symmetry of the Gaussian distribution cannot be broken. If one makes additional assumptions, the causal graph can be recovered with more than one Gaussian source~\citep{rolland_score_2022,montagna_scalable_2023}.
        However, even with Gaussian noise, consistent and asymptotically normal treatment effect estimation is possible with \gls{oml}~\citep{chernozhukov_doubledebiased_2017}. 
        The difference enabling treatment effect estimation but not \gls{bss} is due to knowing the causal graph in treatment effect estimation: with a known causal graph, even the much harder \gls{crl} problem becomes solvable under some circumstances~\citep{wendong_causal_2023}. 
        Knowing the causal graph translates into knowing the 
        structure of the inverse map $\mat{W}$ from observations to sources, \ie, knowing which elements are (non-)zero in $\mat{W}$, and this knowledge in turn breaks the rotational symmetry. 
        Hence, there is no free lunch: the more relaxed conditions on the noise distribution come at the price of knowing the causal graph.

        However, \citet{mackey2018orthogonalICML,jin2025its} showed that, even when one knows the causal graph, one can obtain more accurate and robust estimates of the treatment effect when the treatment noise is non-Gaussian. 
        Hence, non-Gaussianity impacts not just infinite-sample consistency (the typical focus of \gls{ica} analyses) but also finite-sample estimator quality.
        These results can be intuitively summarized as non-Gaussian components are easier to discern, both in \gls{bss} and treatment effect estimation (\cf~\cref{table:breaking_symmetries}).  
        In the sequel, we will exploit this connection to accurately estimate treatment effects with linear \gls{ica}. 

\section{Estimating Treatment Effects with ICA}\label{sec:effect_ica}
Building on the insights of the preceding section, we now turn our attention to designing an \gls{ica} solution to the treatment effect estimation problem.
Inverting the mixing function with \gls{ica} requires detailed knowledge of the \gls{dgp}, \ie, \gls{ica} needs to be able to extract the correct functional relationship up to an equivalence class. \citet{shimizu_linear_2006,reizinger_jacobian-based_2023} demonstrated that recovering the source variables conveys information about the causal structure. Treatment effect estimation, under the prevalent assumptions in the literature, presents a simpler task than recovering the source variables. Namely, it is only a partial reconstruction task (the target quantity is only the causal effect), with more prior knowledge (the causal graph is known). 
        We will show how in this case, ICA can estimate treatment effects, even with Gaussian covariate noise.

    \paragraph{Method summary.}
        Operating under a linear additive model (\cref{def:lin_plr}), we use a two-step process to estimate the treatment effect with ICA (\cref{fig:fig1}):
        \begin{enumerate}[leftmargin=*,nolistsep]
            \item We run the FastICA algorithm to estimate the sources up to scaling and permutation.
            \item We use our knowledge of the graph to resolve the permutation indeterminacy and the unit-variance of the outcome noise to resolve the scaling indeterminacy.
        \end{enumerate}

    \subsection{Linear PLR}

        We first prove that linear \gls{ica} can estimate treatment effects under a linear \gls{plr} model.
            \begin{definition}[Linear PLR]\label{def:lin_plr}
                A \emph{linear \gls{plr}} model with the graph $T{\color{figblue}\to} Y$ and $T{\color{figred}\leftarrow} X {\color{figgreen}\to} Y$ 
                and independent zero-mean sources 
            $(\xi, \eta, \eps)$ 
            has linear dependence on $X$ given by the (inverse) \gls{sem}:
                \begin{align*}
                    \begin{bmatrix}
                        X\\
                        T\\
                        Y
                    \end{bmatrix} \!\!&=\!\!
                    \begin{bmatrix}
                        0 & 0 & 0\\
                        {\color{figred}a} & 0 & 0\\
                        {\color{figgreen}b} & \color{figblue}\theta & 0
                    \end{bmatrix}\! \!
                    \begin{bmatrix}
                        X\\
                        T\\
                        Y
                    \end{bmatrix}
                   \!\! + \!\!
                    \begin{bmatrix}
                        \xi\\
                        \eta\\
                        \varepsilon
                    \end{bmatrix} =
                    \underbrace{
                    \begin{bmatrix}
                        1 & 0 & 0\\
                        {\color{figred}a} & 1 & 0\\
                        {\color{figgreen}b}+{\color{figred}a}\color{figblue}\theta & \color{figblue}\theta & 1
                    \end{bmatrix}
                    }_{=\mat{A}\ \text{(mixing matrix)}}
                    \! \!
                    \!\begin{bmatrix}
                        \xi\\
                        \eta\\
                        \varepsilon
                    \end{bmatrix},
                    \\
                    \begin{bmatrix}
                        \xi\\
                        \eta\\
                        \varepsilon
                    \end{bmatrix} \!\!&=\!\! 
                    \underbrace{
                    \begin{bmatrix}
                        1 & 0 & 0\\
                        {\color{figred}-a} & 1 & 0\\
                        {\color{figgreen}-b} & \color{figblue}-\theta & 1
                    \end{bmatrix}
                    }_{=\mat{W}\ \text{(\textit{un}mixing matrix})}
                    \!\!
                    \begin{bmatrix}
                         X\\
                        T\\
                        Y
                    \end{bmatrix}.
                \end{align*}                
            \end{definition}
            \vspace{-0.5\baselineskip}
    \citet{shimizu_linear_2006} proved that linear \gls{ica} can be used for \gls{cd}, where the unmixing matrix \mat{W} encodes the direct edges between $(X,T,Y)$---already highlighting the connection between the two fields. 
    Under \cref{assum:lin_plr}, linear ICA recovers the source variables up to scaling and permutation. 
    \begin{assum}[Linear ICA for PLR]\label{assum:lin_plr}
        In the linear \gls{plr} model of \cref{def:lin_plr}, we assume:
        \begin{enumerate}[leftmargin=*,nolistsep, label=(\roman*)]
            \item At most one of the source \glspl{rv} is Gaussian. %
            \item There are no latent confounders.
        \end{enumerate}
    \end{assum}

    However, source recovery alone is insufficient for treatment effect estimation, as linear \gls{ica} cannot resolve permutations and scaling. Fortunately, we can use our knowledge of the \gls{plr} causal graph to resolve the permutation~\citep{reizinger_jacobian-based_2023}. Further, the canonical form of the \gls{anm} implies that the noise variables 
    have a scalar factor of one~\citep{hoyer_nonlinear_2008}, which means that we can resolve the scaling as well. 
    We formalize these findings in the following result, proved in \cref{subsec:app_lin_plr_ica}.
    \begin{restatable}[Treatment effect estimation with ICA]{proposition}{lemlinplrica}\label{lem:lin_plr_ica}
        When \cref{assum:lin_plr} holds and $\Var(\eps)=1$,  linear ICA identifies the causal effect $\theta$ at the global optimum of the loss in the infinite sample limit.
    \end{restatable}

    \cref{lem:lin_plr_ica} can be thought of as relaxing the assumptions and contextualizing the results of~\citet{shimizu_linear_2006} to show that \gls{ica} can be used to identify and consistently estimate treatment effects. 
    This opens the door to new lines of inquiry connecting the well-developed methodology of \gls{ica}
    with the inferential goals of causal effect estimation. 
\subsection{FastICA, OML, and asymptotic relative efficiency}\label{subsec:asymp_var}%
As an initial case study, %
we next identify several similarities and differences between our \gls{ica} approach to \gls{plr} treatment effect estimation and
the state-of-the-art \gls{oml} approach. 
When $\E[\eta^2] = 1$, both FastICA~(\cref{alg:fastica}) and higher-order \gls{oml} employ a specific measure of non-Gaussianity $U$, like the excess kurtosis measure $U(\eta) = \eta^4 - 3$. 
FastICA uses $U$ to identify non-Gaussian source variables, and, in our setting, \citet[Thm.~8.1]{hyvarinen_independent_2001} implies that FastICA consistently recovers the treatment noise source $\eta$ if and only if 
    \begin{align}
         \E[\eta\cdot U'(\eta) - U''(\eta)] \neq 0. \label{eq:ica_cond}
    \end{align}
Higher-order \gls{oml} uses $U$ to construct estimating equations that, remarkably, yield consistent, robust estimates of $\theta$ if and only if the same non-Gaussianity condition \cref{eq:ica_cond} is met~\citep[Thm.~9]{mackey2018orthogonalICML}. For a detailed discussion, refer to \cref{sec:app_moment_cond}.
This correspondence allows us to provide a precise comparison between the asymptotic variances of the two treatment effect estimates when the same measure of non-Gaussianity is employed. See \cref{sec:app_asymp_var} for the proof of this result.

    \begin{theorem}[Asymptotic relative efficiency]\label{efficiency} 
    Under the assumptions of~\citet[Thm~4.5]{auddy_large_2023} and~\citet[Thm.~9]{mackey2018orthogonalICML} 
    with $\Var(\eta)=\Var(\eps)=1$, the FastICA (\cref{alg:fastica}) and higher-order \gls{oml} treatment effect estimators with non-Gaussianity measure $U(\eta) = \eta^4 - 3$ are asymptotically normal with 
        \begin{align}
        \textup{AsymptoticVariance}(\hat\theta_{\text{\tiny OML}})
            \!&=\!  
        \textfrac{\Var(\eta^3) - 6 \E[U(\eta)]- 9}{\E[U(\eta)]^2},\\
        \textup{AsymptoticVariance}(\hat\theta_{\text{\tiny ICA}})
            \!&=\!
         \textfrac{((b\! +\! a\theta)^2 + 1)\Var(\eta^3)}{\E[U(\eta)]^2}.\label{eq:asymp_vars}
        \end{align}
    \end{theorem}
    Remarkably, FastICA has smaller asymptotic variance---and hence higher sample efficiency---than the state-of-the-art OML estimator  whenever the confounding effect $b+a\theta$ has small magnitude and the excess kurtosis $\E[U(\eta)]$ is sufficiently negative. 
    Meanwhile, OML has smaller asymptotic variance when $(b+a\theta)^2$ or $\E[U(\eta)]$ is sufficiently large. 
    In \cref{subsec:exp_homl}, we will see these relative efficiencies reflected in our empirical results with FastICA dominating for smaller values of $(b+a\theta)^2$ and OML for larger.
    \subsection{Estimating multiple treatment effects}
        Next we show that ICA can be used to simultaneously estimate the effects of multiple treatments. For illustration, we consider two treatments:

            \begin{definition}[Multiple treatment linear PLR]\label{ex:lin_plr_multi_treatment}
                A linear \gls{plr} model with the graph $T_1{\color{figblue}\to} Y$, $T_2{\color{figblue}\to} Y$, and $T_{1,2}{\color{figred}\leftarrow} X {\color{figgreen}\to} Y$ and independent zero-mean
            $(\xi, \eta_1, \eta_2, \eps)$  is given by the (inverse) \gls{sem}:
                \begin{align*}
                    \begin{bmatrix}
                        X\\
                        T_1\\
                        T_2\\
                        Y
                    \end{bmatrix} \!\!&=\!\!
                    \begin{bmatrix}
                        0 & 0 & 0 & 0\\
                        {\color{figred}a_1} & 0 & 0 & 0\\
                        {\color{figred}a_2} & 0 & 0 & 0\\
                        {\color{figgreen}b} & \color{figblue}\theta_1 &\color{figblue}\theta_2 & 0
                    \end{bmatrix}\! 
                    \begin{bmatrix}
                        X\\
                        T_1\\
                        T_2\\
                        Y
                    \end{bmatrix}
                    +
                    \begin{bmatrix}
                        \xi\\
                        \eta_1\\
                        \eta_2\\
                        \varepsilon
                    \end{bmatrix},
                    \!
                    \\
                    \begin{bmatrix}
                        \xi\\
                        \eta_1\\
                        \eta_2\\
                        \varepsilon
                    \end{bmatrix} \!\!&=\!\! 
                    \begin{bmatrix}
                        1 & 0 & 0 & 0\\
                        {\color{figred}-a_1} & 1 & 0 & 0\\
                        {\color{figred}-a_2} & 0& 1 & 0 \\
                        {\color{figgreen}-b} & \color{figblue}-\theta_1 & \color{figblue}-\theta_2 & 1 
                    \end{bmatrix}\!
                    \begin{bmatrix}
                         X\\
                        T_1\\
                        T_2\\
                        Y
                    \end{bmatrix}.
                \end{align*}                
            \end{definition}

            \begin{restatable}[Estimating multiple treatment effects with ICA]{proposition}{cormultit}\label{cor:multi_T}
                Under \cref{assum:lin_plr} in the multiple treatment linear PLR model (\cref{ex:lin_plr_multi_treatment}), ICA identifies all treatment effects at the global optimum of the loss in the infinite sample limit.
            \end{restatable}
            The proof in \cref{subsec:app_multi_T} parallels that for single treatments.

      \subsection{Tolerating Gaussian covariates}
      Our next result shows that, perhaps surprisingly, ICA can accurately identify treatment effects even when one or more covariates are Gaussian. 
            This result, proved in \cref{subsec:app_cor_ica_gauss}, still uses the non-Gaussianity of both the outcome and treatment noise to distinguish the nuisance effects.
            \begin{restatable}[Treatment effect estimation with Gaussian covariates and ICA]{proposition}{coricagauss}\label{cor:ica_gauss}
                Consider the generalization of the linear PLR model (\cref{def:lin_plr}) with $d$ covariates $X$.
                If $\Var(\eps)=1$ and both $\eps$ and $\eta$ are non-Gaussian, then linear ICA identifies the treatment effect at the global optimum of the loss in the infinite data limit (even if one or more covariates are Gaussian).
            \end{restatable}

    \subsection{Treatment effect estimation under model misspecification: Nonlinear PLR}

        We next investigate the case when the covariates affect treatment and outcome in a nonlinear way. 
        Specifically, we will use insights from the fields of nonlinear \gls{ica} (conditional independence) and score-based \gls{cd} (derivatives for \glspl{anm}, deferred to \cref{sec:app_cd}) to assess the feasibility of using FastICA---which operates under a \textit{linear} ICA model---for treatment effect estimation in the \emph{nonlinear} \gls{plr} model.
        \begin{definition}[Nonlinear PLR]\label{ex:nl_plr}
                 A nonlinear \gls{plr} model with the graph $T{\color{figblue}\to} Y$ and $T{\color{figred}\leftarrow} X {\color{figgreen}\to} Y$ %
                 and independent zero-mean
            $(\xi, \eta, \eps)$ 
                 is given by the (inverse) \gls{sem}:
                 \vspace{-.35em}
                \begin{align*}
                    \begin{bmatrix}
                        X\\
                        T\\
                        Y
                    \end{bmatrix} \!\!&=\!\!
                    \begin{bmatrix}
                       \xi\\
                        {\color{figred}g}(X) + \eta \\ %
                        {\color{figgreen}f}(X) +  {\color{figblue}\theta}T + \varepsilon
                        \
                    \end{bmatrix}\!
                   \!;\!\!\!\!\! &
                    \begin{bmatrix}
                        \xi\\
                        \eta\\
                        \varepsilon
                    \end{bmatrix} \!\!=\!\! 
                    \begin{bmatrix}
                       X \\
                        T -{\color{figred}g}(X) \\
                        Y -{\color{figgreen}f}(X)  -{\color{figblue}\theta} T\\
                    \end{bmatrix}.\!
                \end{align*}                
            \end{definition}
        \vspace{-.5em}

        \paragraph{Insights from nonlinear ICA: conditional independencies in PLR.}
            Many nonlinear ICA methods assume a notion of ``variability'' of the data distribution, which can often be characterized by 
            additional conditional independencies~\citep{guo_causal_2022,guo_finetti_2024,reizinger_identifiable_2024}. 
            Inspired by these results, we apply the lens of conditional independence to the \gls{plr} model.
            By introducing two conditional source variables $\varepsilon'(X), \eta'(X)$---where $X$ plays the role of the auxiliary variable in the nonlinear ICA literature---we can rewrite the nonlinear PLR equations into a form which shows their conditional independence given $X$: %
            \begin{align}
                \varepsilon'(X) &=f(X) + \varepsilon; \qquad
                \eta'(X) = g(X) + \eta;\\
                \eta'(X) &\indep \varepsilon'(X) \mid X. \label{eq:exchg_plr}
                \vspace{-0.5em}
            \end{align}
            As the above conditional independence does not depend on the distribution of $X$, as in \cref{cor:ica_gauss}, $X$ can even have a Gaussian distribution. %
            Moreover, this nonlinear, conditionally independent setup is known to yield identifiability under ``sufficient variability conditions'' that arise, for example, when data are generated from sufficiently different subgroups~\citep{hyvarinen_unsupervised_2016,hyvarinen_nonlinear_2019,wendong_causal_2023,reizinger_identifiable_2024,morioka_connectivity-contrastive_2023}.
            When the identifiability result is up to permutation, scaling, and zero-preserving elementwise nonlinear transformations, then we can recover the causal graph and resolve the permutation indeterminacy by the result of~\citet{reizinger_jacobian-based_2023}. When we have identifiability up to only permutation and scaling, we recover the following Jacobian of the inference map from $Z$ to $S$ up to a non-zero constant $c$:
            \begin{align}
                \gls{jacobian} &= c\cdot \begin{bmatrix}
                      1 & 0 & 0 \\
                        -{\color{figred}g'}(X) & 1 & 0\\
                        -{\color{figgreen}f'}(X) &  -{\color{figblue}\theta} & 1\\
                    \end{bmatrix}.\label{eq:jinf_plr}
                    \vspace{-1.5em}
            \end{align}
            Fortunately, due to the specific additive structure of the \gls{plr} model, we can also recover the treatment effect exactly from this scaled Jacobian by simply dividing the $-\theta$ entry of the Jacobian by any diagonal entry to eliminate the unknown scaling. 
        \paragraph{Using FastICA for nonlinear PLR.}
            The above intuition gives us hope that if we only aim to selectively estimate the treatment effect (and not to solve the full \gls{bss} problem), then the additive structure of \gls{plr} might lead to good treatment effect estimates, even when the covariates affect both treatment and outcome nonlinearly. We explicitly test this using the following setup, \cf \cref{subsec:lin_nonlin}):
            \begin{enumerate}[nolistsep]
                \item We generate data from a nonlinear \gls{plr} model;
                \item We input the observations, \ie, $(X,T,Y)$ into \textit{linear} FastICA; and 
                \item We use the same strategy as in \cref{lem:lin_plr_ica} to resolve the permutation and scaling indetereminacies.
            \end{enumerate}

\section{Experiments}\label{sec:exp}

We now turn to an empirical evaluation to validate the insights from our theory.
We ground our analysis in the demand estimation setting of~\citet{mackey2018orthogonalICML} 
and 
provide open-source Python code replicating all experiments.\footnote{Code: \href{https://github.com/rpatrik96/ica_causal_effect}{\texttt{github.com/rpatrik96/ica\_causal\_effect}}.}

\begin{figure*}[tb]
    \centering
    \begin{minipage}[t]{0.42\textwidth}
        \vspace{0pt}
        \centering
        \includesvg[width=\textwidth]{figures/noise_ablation/coefficient_ablation/rmse_diff_vs_ica_var_coeff.svg}
    \end{minipage}%
    \hfill
    \begin{minipage}[t]{0.55\textwidth}
        \vspace{30pt}
        \centering
        \footnotesize
        \setlength{\tabcolsep}{3.5pt}
        \begin{tabular}{lccc}
        \toprule
        $c_{\text{ICA}}$ Regime & OML RMSE & ICA RMSE & ICA Wins \\
        \midrule
        Low ($< 1.5$) & 0.0195 & \textbf{0.0144} & 52/54 (96.3\%) \\
        Medium ($1.5$--$5$) & \textbf{0.0183} & 0.0240 & 10/28 (35.7\%) \\
        High ($\geq 5$) & 0.0194 & \textbf{0.0177} & 8/14 (57.1\%) \\
        \midrule
        Overall & 0.0191 & \textbf{0.0177} & 70/96 (72.9\%) \\
        \bottomrule
        \end{tabular}
    \end{minipage}
\caption{\textbf{Relative efficiency of ICA vs.\ higher-order OML for demand estimation} (see \cref{subsec:exp_homl}). 
    \textbf{Left:} RMSE difference (ICA $-$ OML) as a function of the ICA asymptotic variance coefficient $c_{\text{ICA}} = 1 + (b + a\theta)^2$ derived in \cref{efficiency}. Blue points indicate ICA outperforms OML; red points indicate OML outperforms ICA. \textbf{Right:} Performance stratified by $c_{\text{ICA}}$ regime. ICA wins overall (72.9\% win rate), dominating especially when $c_{\text{ICA}} < 1.5$ (96.3\% win rate). OML is preferable in the medium regime ($1.5 \leq c_{\text{ICA}} < 5$), with a  64.3\% win rate.} %
    \label{fig:ica_main_result}
\end{figure*}

    \subsection{Demand estimation from price and purchase data}%
    \label{subsec:exp_homl}
    While the partially linear and non-Gaussian assumptions of this work are by no means universal, there are many practical inference scenarios in which both are known to hold. 
    One important example detailed in~\citet{mackey2018orthogonalICML} is demand estimation from purchasing and pricing data. In this setting, the outcome $Y$ is the observed demand for a product, the treatment $T$ is the price of that product, and, commonly, conditional on all observable covariates $X$, the treatment noise follows a discrete (and non-Gaussian) distribution representing random discounts offered to customers over a baseline price linear $g(X)$ in the observable variables.

    \textbf{Assessing relative efficiency.} 
    To validate the relative efficiency theory of \cref{efficiency}, we will compare our FastICA-based effect estimation strategy with the state-of-the-art higher-order \gls{oml}~\citep{mackey2018orthogonalICML} method. 
    We begin by replicating the synthetic demand estimation experiments of~\citet[Sec.~5]{mackey2018orthogonalICML} with  $n=5000$, $d=10$, uniform and unit-variance $\eps$, discrete $\eta$, independent generalized-normal covariates $X$ with $\beta=1$, varying $\theta\in\{0.01, 0.1, 0.5, 1, 3, 10\}$, and varying sparse linear $f(x)$ and $g(x)$ with $s=1$ non-zero coefficient  $a\in\{-0.002, 0.05, -0.43, 1.56\}$ and $b\in\{0.003, -0.02, 0.63, -1.45\}$. 
    We use the \texttt{scikit-learn}~\citep{scikit-learn} implementation of FastICA with a \texttt{logcosh} loss function and \texttt{unit-variance} whitening.
    We use the implementation of~\citet{mackey2018orthogonalICML} for \gls{oml} methods and use the same tolerance ($10^{-4}$) and iteration count ($1000$) for all methods. 

    \begin{figure}[htbp]
        \centering
        \includesvg[width=\columnwidth, keepaspectratio]{figures/rmse_diff_sample_size_vs_beta_filtered_below_1.5}
        \caption{\textbf{RMSE difference (ICA $-$ higher-order OML) as $n$ and covariate distribution $\beta$ vary} for $c_{\text{ICA}} < 1.5$.}
        \label{fig:rmse_diff_vs_beta}
    \end{figure}

    \textbf{Results.}
    As \cref{efficiency} characterizes the regimes in which either FastICA or OML is more sample efficient, we focus on comparing the estimators' relative efficiency here and include absolute performance metrics and extensive ablations in \cref{sec:app_exp}. 
    \Cref{fig:ica_main_result} shows that \gls{ica}'s relative performance depends on the  asymptotic variance coefficient $c_{\text{ICA}} \defeq 1 + |b + a\theta|_2^2$ derived in \cref{efficiency}. Namely, ICA dominates when the asymptotic variance coefficient $c_{\text{ICA}} < 1.5$ (96.3\% win rate) and overall (72.9\% win rate), while OML is preferable in the medium regime ($1.5 \leq c_{\text{ICA}} < 5$), winning 64.3\% of configurations.

    \paragraph{Ablations over  $n,d,$ and $\beta$.}
    When $c_{\text{ICA}} < 1.5$, 
    we further find that \gls{ica} outperforms \gls{oml} across nearly all tested sample sizes $n\in \{100,200,500,1000,2000,5000\}$, covariate dimension $d\in\{2,5,10,20,50\}$, and covariate
    distributions, with the advantage being most pronounced for small sample sizes, higher values of the generalized normal parameter $\beta$ (\ie, for negative excess kurtosis, or equivalently, thinner tails), and lower covariate dimension (see \cref{fig:rmse_diff_vs_beta,fig:filtered_diff_heatmaps} for RMSE and bias difference heatmaps across these dimensions).
    Even more surprisingly, \gls{ica} can comparably or better estimate the treatment effect even if the noises are Gaussian, that is, when source identification is impossible (\cref{fig:filtered_diff_heatmaps}, right column, $\beta=2$). We report the \gls{mcc}~\citep{hyvarinen_unsupervised_2016}, a measure of source identification in \cref{fig:ica_mcc}, demonstrating that low \gls{mcc} values for Gaussian sources do not prohibit treatment effect estimation. The absolute MSE values for ICA are shown in \cref{fig:ica_mse}.
    These empirical findings align with our asymptotic relative efficiency analysis~(\cref{efficiency}): \gls{ica} is more accurate than \gls{oml} when the confounding effect $(b+a\theta)$ is smaller in magnitude.

    \textbf{Additional ablations.} We present additional ablations varying $a,b,\theta,$ and $c_{\text{ICA}}$ in \cref{fig:coeff_scatter,tab:by_treatment_effect}, varying the FastICA loss function in \cref{fig:ica_mse_fun}, and 
    varying the treatment noise $\eta$ distribution and variance 
    in \cref{fig:distribution_heatmap,fig:combined_heatmap,fig:variance_heatmap}.
    Notably, we find in \cref{fig:distribution_heatmap} that ICA continues to outperform OML across a wide range of treatment distributions with varying kurtoses including discrete, Laplace, Uniform, and Rademacher. 

              \begin{figure*}[tb]
                \centering
        	\begin{subfigure}[b]{0.45\textwidth}
                    \includesvg[width=6.5cm, keepaspectratio]
                    {figures/heatmap_dimension_vs_nonlinearity.svg}
                \label{fig:heatmap_dimension_vs_nonlinearity}
        	\end{subfigure}
                \begin{subfigure}[b]{0.45\textwidth}
        		\centering
            		\includesvg[width=6.5cm,keepaspectratio]{figures/heatmap_ica_multi_dimensions_vs_samples_rel_m2}
        	\end{subfigure}
                \caption{ 
                \textbf{Left:} Relative \Gls{mse} of treatment effect estimation for Laplace noises in \textit{nonlinear} \gls{plr} across multiple covariate dimensions for \textit{linear} ICA with different nonlinearities with $5,000$ samples. Leaky ReLU uses a slope of $0.2$. See \cref{fig:heatmap_dimension_vs_slope_leaky_relu} for an ablation over slopes.  
                \textbf{Right:} Relative \gls{mse} of ICA treatment effect estimation across covariate dimensions $d$ and sample sizes $n$ for $m = 2$ treatments in linear \gls{plr}.
                Means calculated from $20$ seeds. See \cref{fig:ica_multi} for an ablation over treatment counts. 
                }
                \vspace{-.35\baselineskip}
                \label{fig:ica_multi_comp}
            \end{figure*}

    \subsection{Linear ICA for nonlinear PLR}\label{subsec:lin_nonlin}
    \textbf{Setup.}
    We next assess the accuracy of linear ICA for treatment effect estimation in the nonlinear PLR model of \cref{ex:nl_plr}. 
    We fix $n=5000$ and $\theta=1.55$, use Laplace noise (with location $0$ and scale $1$) for $\eps,\eta,$ and $X$, vary the covariate dimension $d \in \{2,5,10,20,50\}$ and consider four nonlinearities $\phi$ for the nuisance functions $f(x)=\phi(\inner{w}{x})$ and $g(x)=\phi(\inner{v}{x})$, where $\phi$ is selected from a rectified linear unit (ReLU), a leaky ReLU (with slope $0.2$), a sigmoid, and $\tanh$. Unlike the single non-zero entry per coefficient vector in \cref{subsec:exp_homl}, the treatment coefficient vector $a$ has approximately $30\%$ non-zero entries and the outcome coefficient vector $b$ is fully dense, both drawn from $\mathcal{N}(0,1)$. We use the same FastICA configuration (\texttt{logcosh}, \texttt{unit-variance} whitening) as in \cref{subsec:exp_homl}. We ablate over the value of $\theta$, the slope of the leaky ReLU, the sparsity of the mixing matrix, and the noise density by changing the generalized normal density's $\beta$ parameter ($\beta=1$ is Laplace) in \cref{sec:app_exp}.

            \textbf{Results.}
            In \cref{fig:ica_multi_comp} (left) 
            we report the mean squared relative error, \ie, $|(\theta-\hat{\theta})/\theta|$, of FastICA 
            across $20$ random seeds.
            Remarkably, linear FastICA performs very well even under model misspecification. That is, it exhibits low relative mean squared error in treatment effect estimation with relative error below $5\%$ in most scenarios.  
            While we do not have a precise theory to characterize success within the nonlinear regime, our empirical results showcase promise.
            In \cref{fig:ica_mse_heatmap_loc_scale,fig:ica_mse_vs_beta,fig:heatmap_dimension_vs_slope_leaky_relu}, we perform additional ablations over the location and scale parameters of the generalized normal distribution, and leaky ReLU slopes, observing comparable performance in most settings.
            Performance is rather insensitive to location and scale (relative error is in the $4-8\%$ range), $\beta$ has the same effect as in the linear case (\cf \cref{fig:ica_mse_vs_beta} left and right), with the highest mean relative errors occurring when $\beta\in \{2,2.5,3\}$, \ie, when the densities are (close to) Gaussian---which is $\beta=2$. Perhaps surprisingly, when only the covariates are Gaussian, we achieve lower relative estimation error than when every noise is Laplace. \cref{fig:dim_vs_nonlinearity} provides a comprehensive view of the robustness across dimensions and nonlinearity types, while \cref{fig:mse_vs_support_size_rel} shows the effect of Gaussian covariates and different treatment effect sampling strategies.

    \subsection{Estimating multiple treatment effects}\label{subsec:exp_multi}
    \textbf{Setup.} Finally, we evaluate the effectiveness of ICA for simultaneously estimating multiple treatment effects as in \cref{ex:lin_plr_multi_treatment}, a question that, to the best of our knowledge, has not seen much work using estimators based on higher-order OML. 
    We use the first $m$ treatment effects in the array $\theta=\brackets{1.55, 0.65, -2.45, 1.75, -1.35}$ for $m\in \{1,2,5\}$ and otherwise adopt the setup of \cref{subsec:lin_nonlin}. %

        \textbf{Results.}
        \cref{fig:ica_multi_comp} (right) reports the mean squared relative error $\Vert (\theta - \hat{\theta})/\theta\Vert_2$ 
        across $20$ random seeds for $d=10$.
        Our primary finding is that when samples are abundant, ICA remains a stable and high-quality estimator across multiple treatments; however the quality of the estimator degrades as a function of $m$ when samples are scarce. 
        Additional experiments ablating the covariate dimension can be found in \cref{fig:ica_multi}.

\section{Discussion, Limitations, and Future Work}
\label{sec:discussion}

        This work initiated the study of \acrfull{ica} for treatment effect estimation, focusing on the \acrfull{plr} model and forging connections with the state-of-the-art higher-order \acrfull{oml} approach of~\citet{mackey2018orthogonalICML}.  %
        
        We theoretically characterized how \gls{ica} can consistently estimate single or multiple treatment effects (\cref{lem:lin_plr_ica,cor:multi_T}) and how the ubiquitous non-Gaussianity assumption in \gls{ica} can be relaxed for the covariate noises (\cref{cor:ica_gauss}). We further demonstrated empirically that linear \gls{ica} can accurately estimate treatment effects even in the presence of nonlinear nuisance effects (\cref{subsec:lin_nonlin}). Finally, we related the asymptotic variances for \gls{ica} and \gls{oml} and identified regimes in which \gls{ica} is provably more sample efficient and empirically more accurate. 

        This work also has its limitations, pointing to several important directions 
        for future work including (1) developing a complete theory of linear ICA for nonlinear PLR, (2) studying the robustness of ICA treatment effect estimates to inaccuracies in estimated nuisance effects, and (3) developing ICA-based tools for causal effect estimation beyond the \gls{plr} model. In \cref{sec:directlingam_comparison}, we also compare FastICA with DirectLiNGAM~\citep{shimizu2011directlingam}, finding complementary strengths: DirectLiNGAM excels in low-dimensional dense settings ($d \leq 10$), while FastICA is preferable for sparse, high-dimensional ($d \geq 20$), or heavily non-Gaussian data.
        %
        %
        %
        %
        %
        %
        %

    %
        
        %

\begin{comment}

    \section*{Acknowledgements}

The authors thank Vahid Balazadeh for his insightful comments.
Patrik Reizinger acknowledges his membership in the European Laboratory for Learning and Intelligent Systems (ELLIS) PhD program and thanks the International Max Planck Research School for Intelligent Systems (IMPRS-IS) for its support. This work was supported by the German Federal Ministry of Education and Research (BMBF): Tübingen AI Center, FKZ: 01IS18039A. Wieland Brendel acknowledges financial support via an Emmy Noether Grant funded by the German Research Foundation (DFG) under grant no. BR 6382/1-1 and via the Open Philantropy Foundation funded by the Good Ventures Foundation. Wieland Brendel is a member of the Machine Learning Cluster of Excellence, EXC number 2064/1 – Project number 390727645. This research utilized compute resources at the Tübingen Machine Learning Cloud, DFG FKZ INST 37/1057-1 FUGG. Rahul G. Krishnan gratefully acknowledges support from the Canada Research Chairs Program (CRC-2022-00049) and the Canada CIFAR AI Chairs Program, This research was funded in part by a NFRF Special Call Award (NFRFR-2022-00526) and NSERC Discovery Grant  (RGPIN-2022-04546).

\end{comment}

\begin{contributions}
P.R.\ conceived the idea, developed the theory, implemented the experiments, and wrote the paper.
L.M.\ provided guidance on the theoretical analysis and OML methodology.
W.B.\ and R.G.K.\ supervised the project and provided feedback.
\end{contributions}

\begin{acknowledgements}
The authors thank Vahid Balazadeh for his insightful comments.
Patrik Reizinger acknowledges his membership in the European Laboratory for Learning and Intelligent Systems (ELLIS) PhD program and thanks the International Max Planck Research School for Intelligent Systems (IMPRS-IS) for its support.
This work was supported by the German Federal Ministry of Education and Research (BMBF): T\"ubingen AI Center, FKZ: 01IS18039A.
Wieland Brendel acknowledges financial support via an Emmy Noether Grant funded by the German Research Foundation (DFG) under grant no.\ BR 6382/1-1 and via the Open Philantropy Foundation funded by the Good Ventures Foundation.
Wieland Brendel is a member of the Machine Learning Cluster of Excellence, EXC number 2064/1 -- Project number 390727645.
This research utilized compute resources at the T\"ubingen Machine Learning Cloud, DFG FKZ INST 37/1057-1 FUGG.
\end{acknowledgements}

\bibliography{refs_filtered}

\newpage
\onecolumn

\title{Estimating Treatment Effects with Independent Component Analysis\\(Supplementary Material)}
\maketitle

\appendix
\etocdepthtag.toc{appendix}
\etocsettagdepth{main}{none}
\etocsettagdepth{appendix}{subsection}
\etocsettocstyle{\section*{Contents}}{}
\tableofcontents

\counterwithin{figure}{section}
\counterwithin{table}{section}
\renewcommand{\thefigure}{\Alph{section}.\arabic{figure}}
\renewcommand{\thetable}{\Alph{section}.\arabic{table}}
\clearpage

\begin{algorithm}[ht]
\caption{\textbf{FastICA (deflationary) with generic non-Gaussianity measure $U$.} \citep{hyvarinen_independent_2001}}
\label{alg:fastica}
\DontPrintSemicolon
\SetAlgoLined
\SetKwInput{KwIn}{Inputs}
\SetKwInput{KwOut}{Outputs}

\KwIn{$\mathbf{X}\in\mathbb{R}^{d\times n}$ (columns are observed samples), target components $m\le d$, tolerance $\varepsilon>0$, maximum iterations $T_{\max}$}
\KwOut{$\mathbf{W}\in\mathbb{R}^{m\times d}$ (unmixing matrix), $\mathbf{S}=\mathbf{W}\,\tilde{\mathbf{X}}\in\mathbb{R}^{m\times n}$ (independent components)}

\BlankLine
\textbf{Centering}:\;
$\bar{\mathbf{x}}\leftarrow \frac{1}{n}\,\mathbf{X}\mathbf{1}$;\quad
$\mathbf{X}\leftarrow \mathbf{X}-\bar{\mathbf{x}}\mathbf{1}^\top$ \tcp*{Zero-mean the data}

\textbf{Whitening}:\;
$\mathbf{C}\leftarrow \frac{1}{n}\,\mathbf{X}\mathbf{X}^\top$;\quad
$\mathbf{C}=\mathbf{O}\mathbf{D}\mathbf{O}^\top$ \tcp*{Eigenvalue decomposition of covariance}
$\mathbf{V}\leftarrow \mathbf{D}^{-1/2}\mathbf{O}^\top$;\quad
$\tilde{\mathbf{X}}\leftarrow \mathbf{V}\mathbf{X}$ \tcp*{Whitened data}

\BlankLine
\textbf{Initialize}\; $\mathbf{W}\leftarrow \mathbf{0}_{m\times d}$

\For{$i\leftarrow 1$ \KwTo $m$}{
  Choose random $\mathbf{w}\in\mathbb{R}^d$; \;
  $\mathbf{w}\leftarrow \mathbf{w}/\|\mathbf{w}\|_2$ \tcp*{Unit-norm initialization}

  \For{$t\leftarrow 1$ \KwTo $T_{\max}$}{
    $\mathbf{o}\leftarrow \mathbf{w}^\top\tilde{\mathbf{X}}$ \tcp*{$1\times n$ projected samples}

    \tcp{Fixed-point (Newton) update}
    $\displaystyle \mathbf{w}^{+}\leftarrow \frac{1}{n}\,\tilde{\mathbf{X}}\,U'(\mathbf{o})^\top
      \;-\;\Big(\frac{1}{n}\sum_{j=1}^{n} U''(u_j)\Big)\,\mathbf{w}$ \;

    \tcp{Deflationary orthogonalization (Gram--Schmidt) w.r.t. previous components}
    \For{$k\leftarrow 1$ \KwTo $i-1$}{
      $\mathbf{w}^{+}\leftarrow \mathbf{w}^{+}-(\mathbf{w}^{+\top}\mathbf{w}_k)\,\mathbf{w}_k$
    }

    $\mathbf{w}^{+}\leftarrow \mathbf{w}^{+}/\|\mathbf{w}^{+}\|_2$ \tcp*{Normalization}

    \If{$\big||\mathbf{w}^{+\top}\mathbf{w}|-1\big|<\varepsilon$}{
      \textbf{break}
    }
    $\mathbf{w}\leftarrow \mathbf{w}^{+}$\;
  }

  $\mathbf{w}_i\leftarrow \mathbf{w}^{+}$;\quad
  $\mathbf{W}_{i,:}\leftarrow \mathbf{w}_i^\top$ \tcp*{Store component $i$}
}

\textbf{Recover sources}:\; $\mathbf{S}\leftarrow \mathbf{W}\,\tilde{\mathbf{X}}$ \tcp*{Independent components}
\end{algorithm}

\section{Insights from Score-Based Causal Discovery}\label{sec:app_cd}

            Recent works~\citep{rolland_score_2022,montagna_scalable_2023,montagna_score_2024,montagna_causal_2023} utilized the (Jacobian of) the score function (the derivative of the log-likelihood) for causal discovery.
           Treatment effect estimation can be thought of as generalizing \gls{cd}: the treatment effect informs us about the ``strength'' of a causal effect, whereas \gls{cd} only seeks to determine the presence or absence of the edges.
             We will use this connection to show that if the goal is to solve a partial \gls{bss} problem by recovering only some of the sources then the non-Gaussianity condition can be relaxed on the covariates.
            To this end, consider the log likelihood of the \gls{plr} model \cref{eq:plr} when $\eps$ and $\eta$ are Gaussian (we only use Gaussians for illustration purposes):
            \begin{talign}
               \scriptsize \log p(Z) = -\half(Y-f(X)-\theta T)^2-\half(T-g(X))^2 + \log p(X).
            \end{talign}
        Irrespective of the distribution of $X$, we can differentiate with respect to $T$ and $Y$ to isolate the causal effect:
        \begin{align}
            \vspace{-0.5em}
            \partial_{T}\!\log p(Z)\!\! &=
               g(X) - T + \theta(Y-f(X)-\theta T)\\
                &= -\eta +\theta \varepsilon; \\
                \partial^2_{T,Y}\log p(Z) &= \theta.
        \end{align}
        The intuitive takeaway from the above is that due to the additive structure of the treatment effect, taking derivatives help

\section{Moment Conditions in Higher-Order OML and ICA}\label{sec:app_moment_cond}

        \subsection{Higher-order OML moment condition for whitened data and $r=3$}\label{subsec:app_homl_cond}

         To achieve $\sqrt{n}$-consistency with higher-order robustness to nuisance errors, the higher-order \gls{oml} estimator of~\citet[Thm.~9]{mackey2018orthogonalICML}  requires the following moment condition on the treatment noise $\eta$ that rules out the Gaussian distribution: for some $r\geq 2, r\in \mathbb{N}$,
        \begin{align}
             \expectation{}\brackets{\eta^{r+1}} &\neq r\expectation{}\brackets{\expectation{}\brackets{\eta^2|X}\cdot\expectation{}\brackets{\eta^{r-1}|X}}.\label{eq:homl_cond}
        \end{align}
       The two sides in \eqref{eq:homl_cond} are equal for when $\eta\mid X$ is Gaussian.
    When $\eta$ has unit variance and $r=3$, this  condition is equivalent to non-zero excess kurtosis. 

    \begin{restatable}[Higher-order OML moment condition for whitened data and $r=3$]{lem}{lemhomlcond}\label{lem:homl_cond}
        When the treatment noise is assumed to have zero mean and unit variance, and $r=3$, then \eqref{eq:homl_cond} is equal to $\E(\eta^4)\neq 3,$ \ie, it measures the kurtosis of $\eta$ and rules out a Gaussian.
    \end{restatable}

        \begin{proof}
            The \gls{homl} estimator uses a test function test function $U'(\eta)=\eta^r$ for estimating $\theta$.
            Furthermore, we have the condition that excludes the Gaussian (for $r=3$):\footnote{This is required to fulfil the non-degeneracy condition, \ie, to avoid that the expectation of $\nabla_\theta m$ is 0}
            \begin{align}
                \expectation{}\brackets{\eta^{r+1}} &\neq r\expectation{}\brackets{\expectation{}\brackets{\eta^2|X}\cdot\expectation{}\brackets{\eta^{r-1}|X}}.\\
                \intertext{By assuming $\eta\perp X$,}
                & = r\expectation{}\brackets{\expectation{}\brackets{\eta^2}\cdot\expectation{}\brackets{\eta^{r-1}}}.\\
                \intertext{With the unit variance constraint on $\eta$, we get}
                &= r\expectation{}\brackets{\eta^{r-1}},\\
                \intertext{which, for $r=3$ yields}
                \expectation{}\brackets{\eta^{4}}&\neq 3\expectation{}\brackets{\eta^{2}}.\\
                \intertext{Noting that the RHS is the variance, we can simplify by the whitening assumption:}
                \expectation{}\brackets{\eta^{4}}&\neq 3,
            \end{align}
            \ie, $\eta$ cannot not be a standard normal \gls{rv}
            Since we assumed $\eta \perp X$ and that $\expectation{}\eta^2 = 1$ (unit variance, which is implied by the whitening preprocessing in ICA).
        \end{proof}

         \subsection{ICA moment condition for whitened data and kurtosis loss}\label{subsec:app_ica_cond}

         ICA has a similar condition to higher-order \gls{oml} (\cf \cref{subsec:app_homl_cond}) for the local optima under the constraint that $\norm{\mat{w}}=1$, which ensures that the FastICA gradient is non-zero~\citep[A.8]{hyvarinen_independent_2001}:
    \begin{align}
         \expectation{}\brackets{\eta\cdot U'(\eta) - U''(\eta)} \neq 0, \label{eq:ica_cond_app}
    \end{align}
    where $t$ is a test function and the data is assumed to be whitened (proof in \cref{subsec:app_ica_cond}).
    \begin{restatable}[ICA moment condition for whitened data and kurtosis loss]{lem}{lemicacond}\label{lem:ica_cond}
        Assume a linear ICA model with $\E [U(\eta)] = \E[\eta^4]$ as a loss function, whitened data, and constrain the rows of the unmixing matrix such that  $\norm{\mat{w}}=1$. Then \eqref{eq:ica_cond_app} is equivalent to $\E(\eta^4)\neq 3.$
    \end{restatable}

        \begin{proof}
        To see the connection to ICA, we recall \citep[Thm.~8.1]{hyvarinen_independent_2000}, stating that for the estimated sources, \ie, the local optima of \expectation{U(\hat{\eta})}, where $U$ is generally chosen as $U(\eta)=\eta^4$, the optimality condition of the theorem is:
        \begin{align}
            \expectation{}\brackets{\eta\cdot U'(\eta) - U''(\eta)} \neq 0,\\
            \intertext{which becomes for the kurtosis-based formulation (\ie, when $U(\eta)=\eta^4$ ):}
            \expectation{}\brackets{\eta^4  - 3\eta^2} &\neq 0.\\
            \intertext{Or, equivalently:}
            \expectation{}\brackets{\eta^4} &\neq\expectation{}\brackets{3\eta^2} = 3.
        \end{align}

    \end{proof}

\section{Proofs}\label{sec:app_proofs}

    \subsection{Proof of Lemma~\ref{lem:lin_plr_ica}}
        \label{subsec:app_lin_plr_ica}

            \lemlinplrica*
            \begin{proof}
                We can apply the theory of linear ICA~\citep{shimizu_linear_2006,hyvarinen_independent_2000} to identify the sources (in the infinite data limit) up to scaling and permutation. Then, exploiting that \mat{A} is triangular, we can permute its estimated inverse $\mat{W} = \inv{\mat{A}}$ into a lower triangular form. Thus, by knowing the graph (particularly that $Y$ is a leaf node), we can resolve the permutation indeterminacy. Thus, we have the estimate of $\varepsilon$ and the corresponding row in \mat{W}. ICA is invariant to scaling the rows of \mat{W}; however, assuming a specific form of how $\varepsilon$ affects $Y$ is sufficient to resolve this ambiguity.
                Finally, selecting the entry characterizing the $T\to \varepsilon$ relationship gives us the causal effect $\theta$.
            \end{proof}

    \subsection{Proof of Corollary~\ref{cor:multi_T}}
        \label{subsec:app_multi_T}
            \cormultit*
            \begin{proof}
                 We can apply the theory of linear ICA~\citep{shimizu_linear_2006,hyvarinen_independent_2000} to identify the sources (in the infinite data limit) up to scaling and permutation. Then, exploiting that \mat{A} is triangular and that $Y$ is a leaf node, we can permute its estimated inverse $\mat{W} = \inv{\mat{A}}$ into a lower triangular form.
                 However, as opposed to \cref{lem:lin_plr_ica}, this does not resolve the permutation between the different treatment variables, as permuting $(\eta_1, \eta_2)$ keeps \mat{W} triangular. However, the knowledge of the graph (which variable is which treatment component) enables us to resolve this ambiguity.
                 Thus, we have the estimate of $\varepsilon$ and the corresponding row in \mat{W}. ICA is invariant to scaling the rows of \mat{W}; however, assuming a specific form (a PLR model with unit variance $\varepsilon$) of how $\varepsilon$ affects $Y$ is sufficient to resolve this ambiguity.

                Finally, selecting the entries characterizing the $T_1, T_2\to \varepsilon$ relationship gives us the causal effects $\theta_1, \theta_2$.
            \end{proof}

    \subsection{Proof of Theorem~\ref{efficiency}: Asymptotic relative efficiency}  \label{sec:app_asymp_var}
        We compare the asymptotic variances of both higher-order \gls{oml}~\citep{mackey2018orthogonalICML} and FastICA~\citep{hyvarinen1997fast} estimators. We are able to do this as the asymptotic normality of the treatment effect estimate is implied by the stated assumptions.
        In the following, we will denote the asymptotic variance of the estimated parameter $\hat\theta$ as:
        \begin{align}
            \textup{AsymptoticVariance}(\hat\theta) \defeq \lim_{n\to\infty}\Var(\sqrt{n}(\hat\theta\!-\!\theta)).
        \end{align}

        \subsubsection{Assumptions}

        We include all assumptions of both Thm.~4.5 of~\cite{auddy_large_2023} and those of Thm.~9 of~\cite{mackey2018orthogonalICML} that are required for our asymptotic variance results.

\begin{assum}[Auddy--Yuan ICA distribution class and CLT regime]\label{ass:auddy-yuan}
Let $Z\in\mathbb{R}^d$ denote the observed vector to which we apply ICA.
(In our \gls{plr} application, $d=p+2$ and $Z=(X^\top,T,Y)^\top$.)
Let $\widetilde Z$ be a whitened version of $Z$, e.g.
\[
\widetilde Z \defeq \Cov(Z)^{-1/2} Z,
\qquad\text{so that}\qquad
\Cov(\widetilde Z)=I_d.
\]
Assume $\widetilde Z$ follows a (whitened) ICA model
\[
\widetilde Z = A S,\qquad A\in\mathcal{O}(d),
\]
where $S=(S_1,\dots,S_d)^\top$ has mutually independent coordinates.

\smallskip
\noindent
\textbf{(Source normalization.)}
For every $j\in[d]$, $\E[S_j]=0$ and $\E[S_j^2]=1$.

\smallskip
\noindent
\textbf{(Non-Gaussianity and moment bounds.)}
There exist constants $\epsilon_1>0$ and $M_1,M_2>0$ such that for every $j\in[d]$,
\[
M_1^{-1}\le \bigl|\kappa_4(S_j)\bigr| \le M_1,
\qquad
\E\bigl[|S_j|^{8+\epsilon_1}\bigr]\le M_2,
\]
where $\kappa_4(S_j)\defeq \E[S_j^4]-3$ is the excess kurtosis.

\smallskip
\noindent
\textbf{(Auddy--Yuan Thm.~4.5 additional requirements.)}
For the bilinear-form normal approximation in \citet[Thm.~4.5]{auddy_large_2023}, assume further:
\begin{enumerate}
\item $\epsilon_1\ge 4$ (so the sources have at least $(8+4)=12$ finite moments);
\item the sample $\widetilde Z_1,\dots,\widetilde Z_n$ consists of i.i.d.\ copies of $\widetilde Z$;
\item the high-dimensional sample size regime $n \ge C d^3 (\log d)^2$ holds for a sufficiently large
universal constant $C>0$;
\item for the specific vectors $u,v\in\mathbb{R}^d$ defining the bilinear functional
$u^\top(\widehat A-A)v$, the asymptotic variance $\sigma_{u,v}^2$ from \citet[Thm.~4.5]{auddy_large_2023}
is non-degenerate:
\[
\liminf_{d\to\infty}\sigma_{u,v} > 0.
\]
\end{enumerate}

Equivalently, in the notation of \citet[Thm.~4.5]{auddy_large_2023}, the law of $\widetilde Z$ belongs to
the ICA class
\begin{align*}
\mathcal{P}_{\mathrm{ICA}}(A;\epsilon_1,M_1,M_2)
\defeq
\Bigl\{\mathcal{L}(\widetilde Z): \widetilde Z=AS,\ A\in\mathcal{O}(d),\
(S_j)_{j=1}^d\ \text{indep.},\\
\E[S_j]=0,\ \E[S_j^2]=1,\
M_1^{-1}\le|\kappa_4(S_j)|\le M_1,\
\E|S_j|^{8+\epsilon_1}\le M_2
\Bigr\}.
\end{align*}
\end{assum}

\begin{remark}[Meaning of Assumption~\ref{ass:auddy-yuan}]\label{rem:auddy-yuan-meaning}
Assumption~\ref{ass:auddy-yuan} packages two layers of conditions:

\begin{enumerate}
\item \textbf{Whitening / orthogonal mixing ($A\in\mathcal{O}(d)$).}
Whitening enforces $\Cov(\widetilde Z)=I_d$, which reduces the general ICA ambiguity
to an \emph{orthogonal} mixing matrix $A$.
This is analytically convenient because the unmixing matrix is $A^{-1}=A^\top$.

\item \textbf{Source normalization ($\E[S_j]=0$, $\E[S_j^2]=1$).}
ICA is invariant to componentwise rescaling: replacing $S_j$ by $c_j S_j$ can be absorbed by
rescaling the corresponding column of $A$.
The mean/variance constraints fix this scale indeterminacy (up to sign).

\item \textbf{Uniform non-Gaussianity via kurtosis bounds.}
The lower bound $|\kappa_4(S_j)|\ge M_1^{-1}$ prevents components from being (nearly) Gaussian,
which would make kurtosis-based identification ill-conditioned.
The upper bound $|\kappa_4(S_j)|\le M_1$ prevents extremely heavy-tailed sources from dominating
fourth-order statistics.

\item \textbf{High-moment control ($\E|S_j|^{8+\epsilon_1}\le M_2$).}
This ensures that empirical higher-order statistics (used by kurtosis-based ICA procedures)
concentrate around their population values.
In \citet[Thm.~4.5]{auddy_large_2023} they require the stronger $\epsilon_1\ge 4$ so that
$(8+\epsilon_1)\ge 12$ moments exist, which is used to obtain a Berry--Esseen-type normal
approximation for bilinear forms of the ICA estimator.

\item \textbf{High-dimensional CLT scaling ($n \gtrsim d^3(\log d)^2$).}
Theorem~4.5 is a \emph{high-dimensional} normal approximation bound.
The scaling $n \ge C d^3(\log d)^2$ is a sufficient regime ensuring the non-asymptotic
remainder terms in the normal approximation are $o(1)$ as $d\to\infty$.

\item \textbf{Non-degenerate variance ($\liminf \sigma_{u,v}>0$).}
The bilinear functional $u^\top(\widehat A-A)v$ can be asymptotically degenerate for particular
choices of $(u,v)$ (e.g.\ if the leading variance term cancels).
The condition $\sigma_{u,v}$ bounded away from $0$ ensures the limiting Gaussian has positive
variance, so a meaningful CLT statement applies.
\end{enumerate}
\end{remark}

        \begin{assum}[Assumptions of Thm.~4.5 of \cite{auddy_large_2023}]\label{ass:auddy-yuan-thm45}
        
        \begin{enumerate}[leftmargin=*,nolistsep, label=(\roman*)]
            \item Let $X\in\mathbb{R}^d$ follow a (whitened) ICA model $X = AS,\qquad A\in \mathcal{O}(d),$ where $S=(S_1,\dots,S_d)^\top$ has independent components.
            \item Assume there exist constants $\epsilon\ge 4$ and $M_1,M_2>0$ such that for all $j\in[d]$,
            \begin{align}
        \mathbb{E}[S_j]&=0,\qquad \mathbb{E}[S_j^2]=1,\\
        M_1^{-1}&\le |\kappa_4(S_j)|\le M_1,\qquad
        \mathbb{E}|S_j|^{8+\epsilon}\le M_2,
            \end{align}
        where $\kappa_4(S_j):=\mathbb{E}[S_j^4]-3$ denotes the excess kurtosis.
        \item Let $X_1,\dots,X_n$ be i.i.d.\ copies of $X$.
        \item For the bilinear-form CLT in \citet[Thm.~4.5]{auddy_large_2023}, additionally assume:
        (i) $n \ge C\, d^3(\log d)^2$ for a universal constant $C>0$, and
        (ii) for the specific $u,v\in\mathbb{R}^d$ of interest, the asymptotic variance is nondegenerate,
        \[
        \liminf_{d\to\infty} \sigma_{u,v} > 0,
        \]
        where $\sigma_{u,v}^2 := u^\top A D_v A^\top u$ (with $D_v$ as defined in \citet[Thm.~4.5]{auddy_large_2023}).
        \end{enumerate}
        \end{assum}

    \begin{assum}[Assumptions of Thm.~9 of \cite{mackey2018orthogonalICML}]\label{ass:mackey-thm9}
    
        \begin{enumerate}[leftmargin=*,nolistsep, label=(\roman*)]
            \item (PLR model) Observe $Z=(T,Y,X)$ satisfying
        \begin{align*}
        Y &= \theta_0 T + f_0(X) + \varepsilon,
        & \mathbb{E}[\varepsilon\mid X,T]=0 \ \text{a.s.},\\
        T &= g_0(X) + \eta,
        & \mathbb{E}[\eta\mid X]=0 \ \text{a.s.},
        \end{align*}
        where $(\eta,\varepsilon)$ are unobserved disturbances with distributions independent of
        $(\theta_0,f_0,g_0)$.
        \item (Non-Gaussian treatment residual) Assume there exists $r\in\mathbb{N}$ such that
        \[
        \mathbb{E}[\eta^{r+1}] \neq
        r\,\mathbb{E}\!\left[\mathbb{E}[\eta^2\mid X]\;\mathbb{E}[\eta^{r-1}\mid X]\right].
        \]
        Equivalently (as used by \citet[Thm.~9]{mackey2018orthogonalICML}), the conditional distribution
        of $\eta\mid X$ is not almost surely Gaussian.
        \item Define the second-order orthogonality multi-index set
        \[
        \mathcal{S}:=\{\alpha\in\mathbb{N}^4:\|\alpha\|_1\le 2\}\setminus\{(1,0,0,1),(0,1,0,1)\}.
        \]
        \end{enumerate}

\paragraph{Multi-index notation and the meaning of the second-order orthogonality set.}
In \citet[Thm.~9]{mackey2018orthogonalICML}, the moment function depends on a vector of
\emph{four} nuisance functions
\[
h(x) \equiv (h_1(x),h_2(x),h_3(x),h_4(x))
:= \bigl(q(x),\, g(x),\, \mu_{r-1}(x),\, \mu_r(x)\bigr),
\]
where typically $\mu_k(x) := \mathbb{E}[\eta^k \mid X=x]$ and $q,g$ are nuisance regressions.

For a multi-index $\alpha=(\alpha_1,\alpha_2,\alpha_3,\alpha_4)\in\mathbb{N}^4$, define its total order
$|\alpha|_1 := \sum_{j=1}^4 \alpha_j$.  Let $D^\alpha m$ denote the mixed partial derivative of $m$
with respect to the \emph{nuisance coordinates}:
\[
D^\alpha m(z;\theta;h)
:= \frac{\partial^{|\alpha|_1}}{\partial h_1^{\alpha_1}\partial h_2^{\alpha_2}
\partial h_3^{\alpha_3}\partial h_4^{\alpha_4}}\, m(z;\theta;h),
\]
evaluated at $h=h_0(X)$ (i.e., at the true nuisance values).

\medskip
\noindent
\textbf{$\mathcal{S}$-orthogonality.}
Given a set $\mathcal{S}\subseteq\mathbb{N}^4$, the moment $m$ is called
$\mathcal{S}$-orthogonal (w.r.t.\ $h_0$) if
\[
\mathbb{E}\!\left[D^\alpha m\!\bigl(Z;\theta_0;h_0(X)\bigr)\,\middle|\,X\right]=0,
\qquad \forall\,\alpha\in\mathcal{S}.
\]

\medskip
\noindent
\textbf{The “second-order orthogonality multi-index set” used in Thm.~9.}
Mackey et al.\ use the set
\[
\mathcal{S}
:= \Bigl\{\alpha\in\mathbb{N}^4:\ |\alpha|_1\le 2\Bigr\}
\setminus \{(1,0,0,1),(0,1,0,1)\}.
\]
This means: \emph{all} conditional expectations of mixed nuisance-derivatives of total order
$\le 2$ are required to vanish, \emph{except} the two cross-derivatives that couple the
fourth nuisance coordinate $h_4=\mu_r$ with $h_1=q$ or $h_2=g$.

Equivalently, $\mathcal{S}$-orthogonality with the above $\mathcal{S}$ is the collection of conditions
\begin{align*}
\mathbb{E}\!\left[\frac{\partial m}{\partial h_j}\!\bigl(Z;\theta_0;h_0(X)\bigr)\,\middle|\,X\right] &= 0,
&& j=1,2,3,4,\\
\mathbb{E}\!\left[\frac{\partial^2 m}{\partial h_j^2}\!\bigl(Z;\theta_0;h_0(X)\bigr)\,\middle|\,X\right] &= 0,
&& j=1,2,3,4,\\
\mathbb{E}\!\left[\frac{\partial^2 m}{\partial h_j\,\partial h_k}\!\bigl(Z;\theta_0;h_0(X)\bigr)\,\middle|\,X\right] &= 0,
&& 1\le j<k\le 4,\ (j,k)\notin\{(1,4),(2,4)\}.
\end{align*}
In words: $m$ is “second-order orthogonal” in every nuisance direction and every pairwise
interaction \emph{except} it does not enforce orthogonality for the mixed interactions
$(q,\mu_r)$ and $(g,\mu_r)$.

    \end{assum}

        \subsubsection{Asymptotic variance of higher-order OML}
            We state the asymptotic variance of the higher-order \gls{oml} estimator from \citep[Thm.~9]{mackey2018orthogonalICML}.

            For \textbf{OML}, the asymptotic variance for $\theta$ is (with test function $U$)
            \begin{align}
                 \mathrm{Var}(\theta_{\mathrm{OML}}) &= J^{-1} V J^{-1}, \\
                 J &= \mathbb{E}[\nabla_\theta m] \quad \text{and}  \quad V = Cov(m),\\
                 \nabla_\theta m &= \varepsilon ( U'(\eta) - \mathbb{E}[U'(\eta)] - \eta \mathbb{E}[U''(\eta)] ),\\
                 J &= \mathbb{E}[\eta U'(\eta) - \eta^2 \mathbb{E}[U''(\eta)] ].\\
                 \intertext{For unit variance, this simplifies to}
                  J &= \mathbb{E}[\eta U'(\eta) - U''(\eta)].
              \end{align}

              yielding the asymptotic variance for the outcome noise  $\varepsilon=Y - q(X) - \theta \eta $
              \begin{align}
                    \textup{AsymptoticVariance}(\hat\theta_{\text{\tiny OML}}) &= \dfrac{\mathbb{E}[ \varepsilon^2 ( U'(\eta) - \mathbb{E}[U'(\eta)] - \eta \mathbb{E}[U''(\eta)] )^2] }{ (\mathbb{E}[\eta U'(\eta) - \eta^2 \mathbb{E}[U''(\eta)] ])^2}.\label{eq:var_homl}\\
                    \intertext{As we assumed unit variance for the noises, this yields:}
                    &=  \dfrac{\mathbb{E}[( U'(\eta) - \mathbb{E}[U'(\eta)] - \eta \mathbb{E}[U''(\eta)] )^2] }{ (\mathbb{E}[\eta U'(\eta) - \mathbb{E}[U''(\eta)] ])^2}.\label{eq:var_homl_unit_var}
                \end{align}

    \subsubsection{Asymptotic variance for FastICA} %

    \paragraph{Setup.}

    We consider the ICA model in the context of a partially linear regression (PLR) setup. Let $Z = \mat{B} S$, where $Z \in \mathbb{R}^d$ are the observed signals, $S$ are the independent sources, and $\mat{B} \in \mathbb{R}^{d \times d}$ is the unwhitened mixing matrix. The whitened observations are $X = \Sigma^{-1/2} Z$, where $\Sigma = \mathrm{Cov}(Z)$. Then the whitened mixing matrix is $\mat{A} = \Sigma^{-1/2} \mat{B}$, so $X = \mat{A} S$.

    We want to compute the asymptotic variance of the FastICA estimate of $B_{3,2} = \theta$, the treatment effect in a PLR model.

    \paragraph{Using \citep[Thm~4.5]{auddy_large_2023}.}
    Consider the vectors
    \begin{equation}
        u = \Sigma^{1/2} e_3, \quad v = e_2 \in \reals^d,
    \end{equation}
    which ensure that
    \begin{talign}
    u^\top (\widehat{\mat{A}}-\mat{A})v=e_3^\top (\widehat{\mat{B}}-\mat{B})e_2 = \widehat{B}_{3,2}-B_{3,2}
    = \hat\theta_{ICA} - \theta.
    \end{talign}
    \citet[Thm~4.5]{auddy_large_2023} gives the asymptotic variance of the bilinear form $u^\top(\widehat{\mat{A}} - \mat{A})v$:

    \begin{equation}
        \sqrt{n} \cdot u^\top(\widehat{\mat{A}} - \mat{A})v \xrightarrow{d} \mathcal{N}(0, \sigma^2_{u,v}) \quad \text{with} \quad \sigma^2_{u,v} = u^\top \mat{A} D_v \mat{A}^\top u,
    \end{equation}

    where $\mat{D}_v$ is a diagonal matrix with diagonal entries
    \begin{equation}
        (D_v)_{kk} = \begin{cases}
            \frac{\mathrm{Var}(S_2^3)}{\kappa_4(S_2)^2} & \text{if } k \ne 2, \\
            0 & \text{if } k = 2.
        \end{cases}
    \end{equation}

    With this choice of $u$ and $v$, the asymptotic variance is
    \begin{align}
        \sigma^2_{u,v} &= u^\top \mat{A} \mat{D}_v \mat{A}^\top u = e_3^\top \mat{B} \mat{D}_v \mat{B}^\top e_3 = \sum_{k=1}^d B_{3k}^2 (D_v)_{kk}.
    \end{align}

    Thus,
    \begin{equation}
        \sigma^2_{u,v} = \sum_{k \ne 2} B_{3k}^2 \cdot \frac{\mathrm{Var}(S_2^3)}{\kappa_4(S_2)^2},
    \end{equation}
    where $\kappa_4$ is the excess kurtosis.

    \paragraph{Substituting the PLR parameterization.}

    In our partially linear regression model, the third row of $\mat{B}$ is
    \begin{equation}
        B_{3,:} = \left( b + a \theta,\, \theta,\, 1 \right).
    \end{equation}

    So we obtain by plugging in $S_2=\eta:$
    \begin{align}
    \textup{AsymptoticVariance}(\hat\theta_{\text{\tiny ICA}})
        = \sigma^2_{u,v} &= \left( (b + a\theta)^2 + 1 \right) \cdot \frac{\mathrm{Var}(\eta^3)}{\kappa_4(\eta)^2} = \left( (b + a\theta)^2 + 1 \right) \cdot \frac{\mathrm{Var}(U'(\eta))}{\mathbb{E}(\eta^4-3)^2}.
    \end{align}

    \begin{remark}
        If $a,b$ are vectors, \ie, when $X$ is vector-valued, then the above expression becomes:
         \begin{equation}
        \textup{AsymptoticVariance}(\hat\theta_{\text{\tiny ICA}}) = \left( \Vert b + a\theta\Vert_2^2 + 1 \right) \cdot \frac{\mathrm{Var}(\eta^3)}{\kappa_4(\eta)^2}.
    \end{equation}
    \end{remark}

    \subsection{Proof of Corollary~\ref{cor:ica_gauss}}
                \label{subsec:app_cor_ica_gauss}
                    \coricagauss*
                    \begin{proof}
                    The log-likelihood of observed causal variables is expressed with change-of-variables in terms of the noises:
                    \begin{align*}
                        \log p_Z(Z) &= \log p_S(S) + \logabsdet{\mat{W}},\\
                        \intertext{where \mat{W} has the following structure}
                        \mat{W} & = \begin{bmatrix}
                                \Id{\dim X} & 0 & 0\\
                                {\color{figred}\mat{A}}  & \Id{\dim T} & 0\\
                                {\color{figgreen}\mathbf{b}} & \color{figblue}\boldsymbol{\theta} & 1
                            \end{bmatrix}\!.
                    \end{align*}
                        If the covariates have a Gaussian noise, then any rotation on the block of covariates will maintain the same likelihood---however, this will also change the direct effect coefficients of $X$ on $Y,T$, \ie, $\mat{A}\in \rr{\dim T \times \dim X}, \mathbf{b}\in \rr{1 \times \dim X}$. Importantly, this does not change the treatment effect coefficients $\color{figblue}\boldsymbol{\theta}$. That is, we can define an equivalence class $\mat{W}_O = \mat{W}\mat{O},$ where \mat{O} is a block-orthogonal matrix  $\mat{O}= \diag{\mat{O}_{\dim X}, \Id{\dim T}, 1}$ with $\mat{O}_{\dim X}$ being a $(\dim X\times \dim X)-$dimensional orthogonal map. In this case, the inverse maps solving the \gls{bss} problem will capture the treatment effect. Thus, we can apply the same argument as in \cref{lem:lin_plr_ica}.
                        Since both the treatment and outcome noise are non-Gaussian, they can be separated from the Gaussian covariate block by ICA. Thus, the treatment effect is identifiable.
                    \end{proof}

\section{Compute Usage}\label{subsec:app_compute}
All experiments ran on CPU-only nodes of an HTCondor cluster, each requesting 4 cores and 32\,GB RAM. A single sweep of the full parameter grid comprises approximately 300 parallel jobs for the demand estimation experiments plus 10--20 ablation jobs, with individual runtimes ranging from under one hour (small sample sizes) to roughly two days ($n{=}5{,}000$). Including preliminary and debugging runs throughout the project, we estimate a total compute budget of approximately 50{,}000 CPU-hours. No GPUs were used.

\section{Additional Experiments: Treatment Noise and PLR Coefficient Ablations}\label{sec:app_exp}

This section provides supplementary experimental results that complement the main text, including additional ablation studies over ICA parameters, performance across different data generating processes, and extended comparisons with alternative methods.

\subsection{Demand Estimation Ablations}\label{subsec:demand_ablations}

These figures provide additional details for the demand estimation experiments in \cref{subsec:exp_homl}. \cref{fig:ica_mcc} shows source identification quality measured by \acrfull{mcc}. \cref{fig:ica_mse} shows MSE heatmaps across sample sizes and covariate dimensions or distribution shape parameters. \cref{fig:filtered_diff_heatmaps} shows RMSE and bias differences between ICA and higher-order OML for filtered configurations ($c_{\text{ICA}} < 1.5$; treatment coefficient $a = 0.5$, outcome coefficient $b = 0.05$, $\theta = 1$).

\begin{figure}[tb]
    \centering
    \includesvg[width=7cm, keepaspectratio]{figures/mcc_sample_size_vs_beta_ica_mean}
    \caption{\textbf{Source identification via \acrfull{mcc} for ICA in linear \gls{plr}.} Means from $20$ seeds ($d=10$). \gls{mcc}~\citep{hyvarinen_unsupervised_2016} measures source recovery (0--1; higher is better). Gaussian covariates ($\beta=2$) yield lowest MCC, as predicted by theory.}
    \label{fig:ica_mcc}
\end{figure}

\begin{figure}[tb]
    \centering
    \begin{subfigure}[b]{0.3\textwidth}
        \centering
        \includesvg[width=\textwidth, keepaspectratio]{figures/bias_sample_size_vs_dim_ica_mean}
    \end{subfigure}
    \begin{subfigure}[b]{0.32\textwidth}
        \centering
        \includesvg[width=\textwidth, keepaspectratio]{figures/bias_sample_size_vs_beta_ica_mean}
    \end{subfigure}
    \caption{\textbf{Treatment effect estimation MSE for ICA with multinomial treatment noise in linear \gls{plr}.} Means from $20$ seeds; {\color{blue}blue} = lower, {\color{figred}red} = higher relative MSE. \textbf{Left:} Covariate dimension $d$ vs.\ sample size ($\beta=1$). \textbf{Right:} Generalized normal shape parameter $\beta$ of the covariate noise density vs.\ sample size ($d=10$).}
    \label{fig:ica_mse}
\end{figure}

\begin{figure}[tb]
    \centering
    \begin{subfigure}[b]{0.48\textwidth}
        \centering
        \includesvg[width=\textwidth, keepaspectratio]{figures/rmse_diff_sample_size_vs_dim_filtered_below_1.5}
        \caption{RMSE difference vs.\ covariate dimension $d$}
    \end{subfigure}
    \hfill
    \begin{subfigure}[b]{0.48\textwidth}
        \centering
        \includesvg[width=\textwidth, keepaspectratio]{figures/rmse_diff_sample_size_vs_beta_filtered_below_1.5}
        \caption{RMSE difference vs.\ $\beta$}
    \end{subfigure}

    \vspace{0.5em}

    \begin{subfigure}[b]{0.48\textwidth}
        \centering
        \includesvg[width=\textwidth, keepaspectratio]{figures/bias_diff_sample_size_vs_dim_filtered_below_1.5}
        \caption{Bias difference vs.\ covariate dimension $d$}
    \end{subfigure}
    \hfill
    \begin{subfigure}[b]{0.48\textwidth}
        \centering
        \includesvg[width=\textwidth, keepaspectratio]{figures/bias_diff_sample_size_vs_beta_filtered_below_1.5}
        \caption{Bias difference vs.\ $\beta$}
    \end{subfigure}
    \caption{\textbf{RMSE and bias differences (ICA $-$ higher-order OML) in linear \gls{plr}.} Results filtered to configurations where $c_{\text{ICA}} < 1.5$. Coefficient vectors have a single non-zero entry: treatment coefficient $a = 0.5$, outcome coefficient $b = 0.05$, $\theta = 1$ (yielding $c_{\text{ICA}} = 1 + (b + a\theta)^2 \approx 1.30$). Means from $20$ seeds; {\color{blue}blue} = ICA outperforms higher-order OML, {\color{figred}red} = higher-order OML outperforms ICA. \textbf{Top row:} RMSE differences. \textbf{Bottom row:} Bias differences. \textbf{Left column:} Covariate dimension $d$ vs.\ sample size ($\beta=4$). \textbf{Right column:} Generalized normal shape parameter $\beta$ of the covariate noise density vs.\ sample size ($d=10$).}
    \label{fig:filtered_diff_heatmaps}
\end{figure}

\subsection{Treatment Noise Distribution Ablation Study}
\label{sec:noise_ablation}

We study how the treatment noise distribution $\eta$ affects the relative performance of higher-order OML and ICA estimators.

  \subsubsection{Experimental Setup}

  All experiments use $n = 5{,}000$ samples, $d = 10$ covariates (support size), and 20 seeds per configuration. We employ 2-fold cross-fitting for nuisance function estimation. To ensure robustness, we randomize the model coefficients: $\theta \in [0.001, 5]$, $a \in [-10, 10]$, and $b \in [-0.5, 0.5]$, with 20 random configurations per noise distribution.

  We evaluate four noise distributions for the treatment noise $\eta$ spanning heavy-tailed and bounded regimes ($\kappa$ denotes the excess kurtosis):
  \begin{itemize}
      \item \textbf{Discrete}: Asymmetric distribution with mass at $\{0, -0.5, -2, -4\}$ and probabilities $(0.65, 0.2, 0.1, 0.05)$; $\kappa \approx 4.97$
      \item \textbf{Laplace}: Heavy-tailed with scale $1/\sqrt{2}$ for unit variance; $\kappa = 3.0$
      \item \textbf{Uniform}: Bounded on $[-\sqrt{3}, \sqrt{3}]$; $\kappa \approx -1.2$
      \item \textbf{Rademacher}: Binary $\{-1, +1\}$ with equal probability; $\kappa = -2.0$
  \end{itemize}

  \subsubsection{Results: Noise Distribution Ablation}

  \Cref{fig:distribution_heatmap} provides a detailed heatmap visualization sorted by excess kurtosis. ICA consistently achieves lower RMSE across all tested distributions, with particularly strong performance for heavy-tailed distributions (high kurtosis).

  \subsection{Treatment Noise Variance Ablation Study}
  \label{sec:variance_ablation}

  To further investigate the interplay between distribution shape and scale, we conduct a variance ablation study varying both the generalized normal shape parameter $\beta \in \{0.5, 1.0, 1.5, 2.5, 3.0, 4.0\}$ and variance $\sigma^2 \in \{0.25, 0.5, 1.0, 2.0, 4.0\}$. We exclude $\beta = 2$ (Gaussian) as ICA is not applicable when the noise is Gaussian. The treatment effect is fixed at $\theta = 1.0$ with $n = 5{,}000$ samples and 20 seeds.

  \subsubsection{Results: Variance Ablation}

  Figure~\ref{fig:variance_heatmap} presents the RMSE difference (ICA $-$ higher-order OML) across the $(\beta, \sigma^2)$ grid. Key findings include:

  \begin{itemize}
      \item For the heaviest-tailed distribution ($\beta = 0.5$), ICA outperforms higher-order OML across most variance levels, with the largest gains at low variance.
      \item For intermediate shape parameters ($\beta \in \{1.0, 1.5\}$), results are mixed: some configurations show large ICA advantages (e.g., $\beta = 1.0$, $\sigma^2 = 0.5$), but higher-order OML is competitive or better in other settings.
      \item For light-tailed distributions ($\beta \geq 2.5$), both methods perform similarly at moderate to high variance, while higher-order OML tends to be slightly better at low variance.
  \end{itemize}

  The excess kurtosis $\kappa$ decreases monotonically with $\beta$: from $\kappa = 22.2$ at $\beta = 0.5$ to $\kappa = -0.81$ at $\beta = 4.0$. This confirms that ICA's advantage is most pronounced when the treatment noise has high kurtosis (heavy tails), consistent with its theoretical asymptotic variance depending on $\kappa^{-2}$.

\begin{figure}[htbp]
    \centering
    \includesvg[width=\textwidth]{figures/noise_ablation/variance_ablation/combined_diff_heatmap.svg}
    \caption{\textbf{Estimation metric differences (ICA $-$ higher-order OML) across the treatment noise variance ablation grid.} Panels show RMSE, absolute bias, and standard deviation differences (left to right). Rows correspond to generalized normal shape parameter $\beta$; columns correspond to noise variance $\sigma^2$. Blue cells indicate ICA outperforms higher-order OML; red cells indicate higher-order OML outperforms ICA.}
    \label{fig:variance_heatmap}
\end{figure}

\begin{figure}[htbp]
    \centering
    \includesvg[width=\textwidth]{figures/noise_ablation/distribution_diff_heatmap.svg}
    \caption{\textbf{Heatmap of estimation performance differences (ICA $-$ OML) across treatment noise distributions}, sorted by excess kurtosis $\kappa$. Blue cells indicate ICA outperforms OML; red cells indicate OML outperforms.}
    \label{fig:distribution_heatmap}
\end{figure}

 \subsection{Coefficient Ablation Study}
  \label{sec:coefficient_ablation}

  We conduct systematic coefficient ablation experiments to understand how the structural parameters of the partially linear model affect the relative performance of higher-order OML and ICA estimators. Recall from \cref{efficiency} that the ICA variance coefficient $c_{\text{ICA}} = 1 + \|b + a\theta\|_2^2$ directly scales ICA's asymptotic variance.

  \subsubsection{Experimental Design}

  We systematically vary three coefficient parameters while holding the noise distribution fixed (discrete asymmetric distribution with $\kappa \approx 4.97$):

  \paragraph{Coefficient Grid (Fixed).} In the deterministic ablation, we use:
  \begin{itemize}
      \item Treatment coefficient: $a \in \{-0.002, 0.05, -0.43, 1.56\}$
      \item Outcome coefficient: $b \in \{0.003, -0.02, 0.63, -1.45\}$
      \item Treatment effect: $\theta \in \{0.01, 0.1, 0.5, 1.0, 3.0, 10.0\}$
  \end{itemize}
  This yields $4 \times 4 \times 6 = 96$ configurations. Only the first coefficient in each support is non-zero; remaining coefficients are set to zero.

  \paragraph{Randomized Coefficients.} We also conduct experiments with randomized coefficients---drawn from a uniform distribution---across multiple noise distributions:
  \begin{itemize}
      \item Treatment coefficient: $a \sim \mathcal{U}[-10, 10]$
      \item Outcome coefficient: $b \sim \mathcal{U}[-0.5, 0.5]$
      \item Treatment effect: $\theta \sim \mathcal{U}[0.001, 0.2]$
  \end{itemize}
  with 20 random configurations per noise distribution.

  \paragraph{Common Settings.} All experiments use $n = 5{,}000$ samples, $d = 10$ covariates, 20 seeds, 2-fold cross-fitting, and Lasso regularization (for OML) with $\lambda = \sqrt{\log(d)/n}$.

  \subsubsection{The Coefficient Cancellation Effect}

  A key finding is the \emph{coefficient cancellation} phenomenon: when $b + a\theta \approx 0$, the ICA variance coefficient approaches its minimum value of 1, making ICA highly efficient. Table~\ref{tab:cancellation_examples} shows examples where near-perfect cancellation occurs.

  \begin{table}[htbp]
  \centering
  \caption{\textbf{Examples of coefficient cancellation where $c_{\text{ICA}} \approx 1$.}}
  \label{tab:cancellation_examples}
  \begin{tabular}{rrrr}
  \toprule
  $a$ & $b$ & $\theta$ & $c_{\text{ICA}}$ \\
  \midrule
  $-0.002$ & $0.003$ & $1.00$ & $1.000001$ \\
  $0.050$ & $-0.020$ & $0.50$ & $1.000025$ \\
  $-0.430$ & $0.003$ & $0.01$ & $1.000002$ \\
  $0.050$ & $-0.020$ & $0.10$ & $1.000049$ \\
  \bottomrule
  \end{tabular}
  \end{table}

  Conversely, when coefficients reinforce rather than cancel, $c_{\text{ICA}}$ can become very large. For instance, with $a = 1.56$, $b = 0.003$, and $\theta = 10$, we obtain $c_{\text{ICA}} = 244.45$.

  \subsubsection{Results: Performance Stratified by ICA Variance Coefficient}

  Figure~\ref{fig:rmse_vs_ica_var} shows how RMSE depends on the ICA variance coefficient. \Cref{fig:ica_main_result} right (in the main text, \cref{sec:exp}) summarizes performance across different $c_{\text{ICA}}$ regimes, showing that while ICA wins overall (72.9\%), OML is preferable in the medium $c_{\text{ICA}}$ regime.

\begin{figure}[htbp]
    \centering
    \begin{subfigure}[b]{0.48\textwidth}
        \centering
        \includesvg[width=\textwidth]{figures/noise_ablation/coefficient_ablation/rmse_vs_ica_var_coeff.svg}
        \caption{RMSE vs.\ ICA variance coefficient $c_{\text{ICA}} = 1 + \|b + a\theta\|_2^2$ on log-log scale. Blue circles: higher-order OML; orange squares: ICA. ICA achieves substantially lower RMSE when $c_{\text{ICA}}$ is small (coefficient cancellation regime).}
        \label{fig:rmse_vs_ica_var}
    \end{subfigure}
    \hfill
    \begin{subfigure}[b]{0.48\textwidth}
        \centering
        \includesvg[width=\textwidth]{figures/noise_ablation/coefficient_ablation/rmse_diff_vs_ica_var_coeff.svg}
        \caption{RMSE difference (ICA $-$ higher-order OML) vs.\ ICA variance coefficient. Blue points indicate ICA outperforms higher-order OML; red points indicate higher-order OML outperforms ICA. The transition occurs around $c_{\text{ICA}} \approx 1.5$.}
        \label{fig:rmse_diff_vs_ica_var}
    \end{subfigure}
    \caption{\textbf{Effect of the ICA variance coefficient on estimator performance.}}
    \label{fig:ica_var_coeff_analysis}
\end{figure}

  Key observations:
  \begin{itemize}
      \item \textbf{Low $c_{\text{ICA}}$ regime}: ICA strongly dominates with 96.3\% win rate, achieving 26\% lower RMSE on average.
      \item \textbf{Medium $c_{\text{ICA}}$ regime}: Higher-order OML becomes preferable, winning 64.3\% of configurations.
      \item \textbf{High $c_{\text{ICA}}$ regime}: Mixed results; ICA can still perform well despite high variance coefficient if other factors are favorable.
  \end{itemize}

  \subsubsection{Results: Randomized Coefficients Across Distributions}

  Table~\ref{tab:coeff_randomized} shows results for randomized coefficient experiments across six noise distributions.

  \begin{table}[htbp]
  \centering
  \caption{\textbf{Randomized coefficient ablation across noise distributions.} ($n = 5{,}000$, 20 configs $\times$ 20 replications each). Coefficients drawn from $a \in [-10, 10]$, $b \in [-0.5, 0.5]$, $\theta \in [0.001, 0.2]$.}
  \label{tab:coeff_randomized}
  \begin{tabular}{lrrrrrr}
  \toprule
  Distribution & $\kappa$ & OML RMSE & ICA RMSE & $|\Delta|$ Bias & Winner \\
  \midrule
  Discrete & 4.97 & 0.0191 & \textbf{0.0148} & 0.4 & ICA \\
  Laplace & 3.06 & 0.0362 & \textbf{0.0161} & 0.0 & ICA \\
  Gennorm ($\beta{=}4$) & $-0.80$ & 0.0292 & \textbf{0.0163} & 1.4 & ICA \\
  Uniform & $-1.19$ & 0.0172 & \textbf{0.0125} & 0.3 & ICA \\
  Rademacher & $-2.00$ & 0.0148 & \textbf{0.0127} & 0.2 & ICA \\
  \bottomrule
  \end{tabular}
  \end{table}

  ICA consistently outperforms higher-order OML across all distributions when averaging over randomized coefficient configurations, with particularly strong advantages for heavy-tailed distributions (Laplace).

  \subsubsection{Dependence on Treatment Effect Magnitude}

  Figure~\ref{fig:rmse_diff_boxplot} shows how performance varies with the treatment effect $\theta$, with \cref{tab:by_treatment_effect} providing detailed statistics. Larger treatment effects tend to increase $c_{\text{ICA}}$ (since $\|b + a\theta\|^2$ grows quadratically in $\theta$ when $a \neq 0$), which can degrade ICA's relative performance. Additional visualizations are provided in \cref{fig:combined_heatmap,fig:coeff_scatter}, showing RMSE and bias differences across noise distributions and coefficient configurations.

  \begin{table}[htbp]
  \centering
  \caption{\textbf{Performance by treatment effect magnitude.} (fixed coefficient ablation, 16 configurations per $\theta$ value).}
  \label{tab:by_treatment_effect}
  \begin{tabular}{rrrr}
  \toprule
  $\theta$ & Avg.\ $c_{\text{ICA}}$ & ICA Wins & Median RMSE Diff \\
  \midrule
  0.01 & 1.62 & 12/16 (75\%) & $-0.0023$ \\
  0.10 & 1.62 & 12/16 (75\%) & $-0.0018$ \\
  0.50 & 1.73 & 12/16 (75\%) & $-0.0015$ \\
  1.00 & 2.16 & 11/16 (69\%) & $-0.0008$ \\
  3.00 & 7.15 & 11/16 (69\%) & $-0.0005$ \\
  10.00 & 65.92 & 12/16 (75\%) & $-0.0012$ \\
  \bottomrule
  \end{tabular}
  \end{table}

  \subsection{Summary}

  The coefficient ablation experiments reveal that ICA's performance is governed by the ICA variance coefficient $c_{\text{ICA}} = 1 + \|b + a\theta\|_2^2$:
  \begin{itemize}
      \item When coefficients nearly cancel ($c_{\text{ICA}} \approx 1$), ICA achieves substantially lower variance than higher-order OML.
      \item The transition region $c_{\text{ICA}} \in [1.5, 5]$ marks where higher-order OML becomes competitive.
      \item Higher-order OML exhibits more stable performance across configurations (std = 0.0031) compared to ICA (std = 0.0076).
      \item The treatment effect magnitude affects $c_{\text{ICA}}$ quadratically, but ICA remains competitive across the tested range.
  \end{itemize}

  \begin{figure}[htbp]
      \centering
      \includesvg[width=0.7\textwidth]{figures/noise_ablation/coefficient_ablation/rmse_diff_boxplot_by_te.svg}
      \caption{\textbf{Box plots of RMSE difference (ICA $-$ higher-order OML) by treatment effect.} Negative values indicate ICA advantage. The median difference remains negative across all treatment effect values.}
      \label{fig:rmse_diff_boxplot}
  \end{figure}

  \begin{figure}[htbp]
      \centering
      \includesvg[width=\textwidth]{figures/noise_ablation/heatmap_combined_n20_tc-10.0to10.0_oc-0.5to0.5_te0.0to0.2.svg}
      \caption{\textbf{Combined heatmap showing RMSE, bias, and standard deviation differences across all randomized coefficient configurations.} Organized by treatment noise distribution $\eta$ and outcome coefficient bin. }
      \label{fig:combined_heatmap}
  \end{figure}

  \begin{figure}[htbp]
      \centering
      \includesvg[width=\textwidth]{figures/coeff_scatter_rmse_grid_n20_tc-10.0to10.0_oc-0.5to0.5_te0.0to5.0}
       \caption{\textbf{RMSE difference grid for randomized coefficients.} ($a \in [-10, 10]$, $b \in [-0.5, 0.5]$, $\theta \in [0, 5]$). Each panel shows RMSE difference vs.\ a different variable: treatment effect $\theta$ (top-left), treatment coefficient $a$ (top-right), outcome coefficient $b$ (bottom-left), and ICA variance coefficient (bottom-right). Blue: ICA better; red: higher-order OML better.}
      \label{fig:coeff_scatter}
  \end{figure}

\subsection{Nonlinear PLR Ablations}\label{subsec:nonlinear_plr_ablations}

These figures provide additional details for the nonlinear PLR experiments in \cref{subsec:lin_nonlin}.

\begin{figure}[tb]
    \centering
    \includesvg[width=0.7\textwidth, keepaspectratio]{figures/ica_mse_heatmap_loc_scale_n5000.svg}
    \caption{\textbf{Effect of location-scale transformation on treatment effect estimation in nonlinear \gls{plr}.} Shows MSE across different transformation parameters.}
    \label{fig:ica_mse_heatmap_loc_scale}
\end{figure}

\begin{figure}[tb]
    \centering
    \includesvg[width=0.7\textwidth, keepaspectratio]{figures/ica_mse_vs_beta_nonlinear_n5000.svg}
    \caption{\textbf{MSE vs.\ distribution shape parameter $\beta$ in nonlinear \gls{plr}.} Shows performance across generalized normal distributions with varying tail heaviness.}
    \label{fig:ica_mse_vs_beta}
\end{figure}

\begin{figure}[tb]
    \centering
    \includesvg[width=0.7\textwidth, keepaspectratio]{figures/heatmap_dimension_vs_slope_leaky_relu.svg}
    \caption{\textbf{Effect of leaky ReLU slope and covariate dimension on treatment effect estimation.} Shows MSE heatmap across different nonlinearity strengths and dimensionalities.}
    \label{fig:heatmap_dimension_vs_slope_leaky_relu}
\end{figure}

\begin{figure}[tb]
    \centering
    \includesvg[width=0.7\textwidth, keepaspectratio]{figures/heatmap_dimension_vs_nonlinearity.svg}
    \caption{\textbf{Robustness across covariate dimensions and nonlinearity types.} Shows relative MSE of treatment effect estimation for linear ICA across different covariate dimensions and nonlinear transformations in the \gls{plr} model.}
    \label{fig:dim_vs_nonlinearity}
\end{figure}

\begin{figure}[tb]
    \centering
    \begin{subfigure}[b]{0.48\textwidth}
        \centering
        \includesvg[width=\textwidth,keepaspectratio]{figures/ica_mse_vs_noise_split_nonlinear_n5000.svg}
    \end{subfigure}
    \begin{subfigure}[b]{0.48\textwidth}
        \centering
        \includesvg[width=\textwidth,keepaspectratio]{figures/ica_mse_vs_theta_choice_nonlinear_n5000.svg}
    \end{subfigure}
    \caption{\textbf{Effect of (Left) Gaussian covariates and (Right) treatment effect sampling on relative MSE with linear ICA in nonlinear \gls{plr}.} Mean $\pm$ std from $20$ seeds. Left: ``true'' = Gaussian $X$, Laplace $T,Y$. Right: ``fixed'' = $\theta=1.55$; ``uniform'' = $\theta \sim U[0,1)$; ``normal'' = $\theta \sim \mathcal{N}(0,1)$.}
    \label{fig:mse_vs_support_size_rel}
\end{figure}

\subsection{Multiple Treatment Ablations}\label{subsec:multi_treatment_ablations}

These figures provide additional details for the multiple treatment experiments in \cref{subsec:exp_multi}.

\begin{figure}[tb]
    \centering
    \includesvg[width=\textwidth, keepaspectratio]{figures/heatmap_ica_multi_dimensions_vs_samples_rel.svg}
    \caption{\textbf{ICA relative MSE: covariate dimension vs.\ sample size for $m \in \{1, 2, 5\}$ treatments.} Each panel shows the mean relative MSE of the ICA treatment effect estimate across $20$ seeds. The shared colorscale is clipped to the 2nd--98th percentile for readability (the outlier at $d{=}50$, $n{=}200$, $m{=}2$ reaches $\approx 44$). Settings: Laplace covariates ($\beta = 1$), $30\%$ sparsity in the mixing matrix.}
    \label{fig:ica_multi}
\end{figure}

\subsection{FastICA Ablations}\label{subsec:ablations}

    \paragraph{FastICA loss function.}
    \cref{fig:ica_mse_fun} compares FastICA loss functions for treatment effect estimation. The \texttt{logcosh} loss outperforms \texttt{cube} (which corresponds to optimizing for excess kurtosis) and matches \texttt{exp}, motivating our default choice. Settings: $\theta=1.55$, $d=50$, single treatment, $n=5{,}000$, Laplace covariates ($\beta=1$), coefficients drawn from $\mathcal{N}(0,1)$ with $30\%$ sparsity ($\approx 15$ non-zeros out of $d=50$), averaged over $20$ seeds.

    \begin{figure}[htb]
        \centering
        \includesvg[width=0.5\textwidth,keepaspectratio]{figures/ica_mse_fun.svg}
        \caption{\textbf{MSE across FastICA loss functions.} ($d=50$, $n=5{,}000$, single treatment). Mean $\pm$ std from $20$ seeds.}
        \label{fig:ica_mse_fun}
    \end{figure}

    \paragraph{Sparsity of the DGP.}
    We investigate how sparsity of the coefficient matrix $\mat{A}$ (covariate $\to$ treatment) affects estimation. Sparsity is controlled via Bernoulli masking probability. \cref{fig:ica_mse_vs_dim_sparsity} shows MSE remains stable across sparsity levels; we use $0.4$ throughout. Settings: $d=50$, single treatment, $n=5{,}000$, $20$ seeds. Additionally, \cref{fig:mse_vs_support_size_rel} examines the effect of Gaussian covariates versus Laplace noise and different treatment effect sampling strategies on relative MSE.

    \begin{figure}[htb]
        \centering
        \includesvg[width=0.5\textwidth,keepaspectratio]{figures/ica_mse_vs_dim_sparsity.svg}
        \caption{\textbf{MSE vs.\ sparsity of $\mat{A}: X \to T$.} ($d=50$, $n=5{,}000$). Mean $\pm$ std from $20$ seeds.}
        \label{fig:ica_mse_vs_dim_sparsity}
    \end{figure}

\subsection{FastICA vs.\ DirectLiNGAM Comparison}
\label{sec:directlingam_comparison}

  We conduct systematic comparisons between FastICA and DirectLiNGAM~\citep{shimizu2011directlingam} for causal treatment effect estimation. While both methods exploit non-Gaussianity for identification, they differ fundamentally: FastICA performs blind source separation to recover independent components, whereas DirectLiNGAM directly estimates the causal graph structure via recursive regression and independence testing.

  \subsubsection{Data Generating Process}

  We generate data according to a nonlinear structural equation model:
  \begin{align}
      X_i &= S_i, \quad i = 1, \ldots, d, \label{eq:covariates_ica} \\
      T_j &= S_{d+j} + \sum_{i=1}^{d} A_{ji} \cdot \sigma(X_i), \quad j = 1, \ldots, k, \label{eq:treatment_ica} \\
      Y &= S_{d+k+1} + \sum_{j=1}^{k} \theta_j T_j + \sum_{i=1}^{d} B_i \cdot \sigma(X_i), \label{eq:outcome_ica}
  \end{align}
  where $S \in \mathbb{R}^{d+k+1}$ are mutually independent source signals drawn from a generalized normal distribution with shape parameter $\beta$, $\sigma(\cdot)$ denotes the leaky ReLU activation function, and $A \in \mathbb{R}^{k \times d}$ is a sparse coefficient matrix with sparsity controlled by probability $p_{\text{sparse}}$. The treatment effects $\theta = (1.55, 0.65, -2.45, 1.75, -1.35)^\top$ are fixed across experiments.

  \subsubsection{Experimental Design}

  We compare FastICA and DirectLiNGAM across five experimental dimensions, each with 20 random seeds for statistical stability:
  \begin{enumerate}
      \item \textbf{Sparsity}: $p_{\text{sparse}} \in \{0.0, 0.1, \ldots, 0.9\}$ with $n = 5{,}000$, $d = 50$, $k = 1$
      \item \textbf{Distribution shape}: $\beta \in \{0.5, 1.0, \ldots, 5.0\}$ with $n = 200$, $d = 10$, $k = 1$
      \item \textbf{Sample size}: $n \in \{100, 200, 500, 1000, 2000, 5000\}$ with $d = 10$, $k = 1$
      \item \textbf{Covariate dimension}: $d \in \{2, 5, 10, 20, 50\}$ with $n = 1{,}000$, $k = 1$
      \item \textbf{Number of treatments}: $k \in \{1, 2, 3, 4, 5\}$ with $n = 1{,}000$, $d = 10$
  \end{enumerate}

  \subsubsection{Results: FastICA vs.\ DirectLiNGAM Comparison}

  \paragraph{Results: Sparsity Ablation}

  Table~\ref{tab:sparsity_comparison} and \cref{fig:sparsity_comparison} present results across sparsity levels. DirectLiNGAM achieves substantially lower MSE for dense models ($p_{\text{sparse}} \leq 0.4$), with up to $4.7\times$ improvement at $p_{\text{sparse}} = 0$. However, FastICA becomes competitive for sparse models ($p_{\text{sparse}} \geq 0.5$), achieving lower MSE while being approximately $270\times$ faster on average (0.56s vs.\ 151.8s).

  \begin{table}[htbp]
  \centering
  \caption{\textbf{FastICA vs.\ DirectLiNGAM across sparsity levels.} ($n = 5{,}000$, $d = 50$, $k = 1$, 20 seeds).}
  \label{tab:sparsity_comparison}
  \begin{tabular}{crrrrrr}
  \toprule
  $p_{\text{sparse}}$ & \multicolumn{2}{c}{MSE} & \multicolumn{2}{c}{Runtime (s)} & Winner \\
  \cmidrule(lr){2-3} \cmidrule(lr){4-5}
   & FastICA & DirectLiNGAM & FastICA & DirectLiNGAM & \\
  \midrule
  0.0 & 0.064 & \textbf{0.014} & 1.40 & 365.5 & DirectLiNGAM \\
  0.1 & 0.091 & \textbf{0.034} & 1.08 & 359.1 & DirectLiNGAM \\
  0.2 & 0.085 & \textbf{0.076} & 0.94 & 277.3 & DirectLiNGAM \\
  0.3 & 0.114 & \textbf{0.068} & 0.40 & 149.7 & DirectLiNGAM \\
  0.4 & 0.122 & \textbf{0.100} & 0.37 & 78.6 & DirectLiNGAM \\
  0.5 & \textbf{0.085} & 0.118 & 0.42 & 61.6 & FastICA \\
  0.6 & \textbf{0.085} & 0.154 & 0.39 & 62.8 & FastICA \\
  0.7 & \textbf{0.109} & 0.148 & 0.18 & 52.3 & FastICA \\
  0.8 & \textbf{0.102} & 0.148 & 0.19 & 55.4 & FastICA \\
  0.9 & \textbf{0.079} & 0.154 & 0.22 & 55.7 & FastICA \\
  \bottomrule
  \end{tabular}
  \end{table}

  \begin{figure}[htbp]
      \centering
      \begin{subfigure}[b]{0.48\textwidth}
          \includesvg[width=\textwidth]{figures/ica/ica_vs_directlingam_sparsity_n5000.svg}
          \caption{MSE comparison}
          \label{fig:sparsity_mse}
      \end{subfigure}
      \hfill
      \begin{subfigure}[b]{0.48\textwidth}
          \includesvg[width=\textwidth]{figures/ica/ica_vs_directlingam_sparsity_runtime_n5000.svg}
          \caption{Runtime comparison}
          \label{fig:sparsity_runtime}
      \end{subfigure}
      \caption{\textbf{FastICA vs.\ DirectLiNGAM across sparsity levels.} ($n = 5{,}000$, $d = 50$). (a) Mean squared error with standard error bars. DirectLiNGAM dominates for dense models ($p_{\text{sparse}} \leq 0.4$); FastICA excels for sparse models. (b) Runtime comparison showing DirectLiNGAM's significant computational overhead.}
      \label{fig:sparsity_comparison}
  \end{figure}

%

  \paragraph{Results: Dimensionality Scaling}

  The relative performance depends critically on the covariate dimension $d$ (Table~\ref{tab:dimension_comparison} and \cref{fig:dimension_comparison}). DirectLiNGAM dominates for low-dimensional settings ($d \leq 10$), achieving up to $2.7\times$ lower MSE at $d = 2$. However, FastICA becomes preferable for high-dimensional data ($d \geq 20$), where DirectLiNGAM's MSE degrades while its runtime increases dramatically (35.5s at $d = 50$ vs.\ 0.22s for FastICA).

  \begin{table}[htbp]
  \centering
  \caption{\textbf{FastICA vs.\ DirectLiNGAM across covariate dimensions.} ($n = 1{,}000$, $k = 1$, 20 seeds).}
  \label{tab:dimension_comparison}
  \begin{tabular}{crrrrrr}
  \toprule
  $d$ & \multicolumn{2}{c}{MSE} & \multicolumn{2}{c}{Runtime (s)} & Winner \\
  \cmidrule(lr){2-3} \cmidrule(lr){4-5}
   & FastICA & DirectLiNGAM & FastICA & DirectLiNGAM & \\
  \midrule
  2 & 0.037 & \textbf{0.014} & 0.007 & 0.027 & DirectLiNGAM \\
  5 & 0.059 & \textbf{0.014} & 0.006 & 0.101 & DirectLiNGAM \\
  10 & 0.061 & \textbf{0.024} & 0.012 & 0.567 & DirectLiNGAM \\
  20 & \textbf{0.082} & 0.086 & 0.032 & 2.852 & FastICA \\
  50 & \textbf{0.144} & 0.231 & 0.224 & 35.524 & FastICA \\
  \bottomrule
  \end{tabular}
  \end{table}

  \begin{figure}[htbp]
      \centering
      \begin{subfigure}[b]{0.48\textwidth}
          \includesvg[width=\textwidth]{figures/ica/ica_vs_directlingam_n_covariates_n1000.svg}
          \caption{MSE vs.\ dimension}
          \label{fig:dimension_mse}
      \end{subfigure}
      \hfill
      \begin{subfigure}[b]{0.48\textwidth}
          \includesvg[width=\textwidth]{figures/ica/ica_vs_directlingam_n_covariates_runtime_n1000.svg}
          \caption{Runtime vs.\ dimension}
          \label{fig:dimension_runtime}
      \end{subfigure}
      \caption{\textbf{FastICA vs.\ DirectLiNGAM across covariate dimensions.} ($n = 1{,}000$, $k = 1$). (a) DirectLiNGAM dominates for $d \leq 10$; FastICA becomes preferable for $d \geq 20$. (b) DirectLiNGAM runtime scales cubically with dimension.}
      \label{fig:dimension_comparison}
  \end{figure}

%

  \paragraph{Results: Distribution Shape, Sample Size, and Multiple Treatments}

  Three additional ablations show broadly consistent patterns (\cref{tab:collapsed_comparison}).
  Across all tested \textbf{sample sizes} ($n \in \{100, \ldots, 5{,}000\}$) and \textbf{number of treatments} ($k \in \{1, \ldots, 5\}$), DirectLiNGAM uniformly achieves $2$--$2.4\times$ lower MSE than FastICA with moderate runtime overhead ($3$--$4\times$ slower).
  For \textbf{distribution shape} ($\beta \in \{0.5, \ldots, 5.0\}$), FastICA wins only for heavily non-Gaussian sources ($\beta = 0.5$, MSE $0.052$ vs.\ $0.120$); DirectLiNGAM dominates for all $\beta \geq 1$.

  \begin{table}[htbp]
  \centering
  \caption{\textbf{Summary of three ablations where DirectLiNGAM wins broadly.} Ranges report (min, max) MSE across the swept parameter; $k = 1$ and $\beta = 1$ unless swept. All experiments use 20 seeds.}
  \label{tab:collapsed_comparison}
  \begin{tabular}{lcccl}
  \toprule
  Ablation & Swept range & FastICA MSE & DirectLiNGAM MSE & Winner \\
  \midrule
  Distribution ($\beta$) & $0.5$--$5.0$ & 0.052--1.863 & 0.106--0.506 & Mixed$^*$ \\
  Sample size ($n$) & $100$--$5{,}000$ & 0.031--0.596 & 0.014--0.249 & DirectLiNGAM \\
  Treatments ($k$) & $1$--$5$ & 0.054--0.161 & 0.033--0.079 & DirectLiNGAM \\
  \bottomrule
  \multicolumn{5}{l}{\footnotesize $^*$FastICA wins only at $\beta = 0.5$ (heavy-tailed); DirectLiNGAM dominates for $\beta \geq 1$.}
  \end{tabular}
  \end{table}

\subsection{Summary}

  Our comprehensive comparison reveals complementary strengths:
  \begin{itemize}
      \item \textbf{DirectLiNGAM} achieves superior accuracy for low-dimensional, dense, moderately non-Gaussian settings, but scales poorly with dimension ($O(d^3)$ complexity).
      \item \textbf{FastICA} excels for sparse, high-dimensional, or heavily non-Gaussian data, with consistent sub-second runtime regardless of problem structure.
      \item The crossover point occurs around $d \approx 15$--$20$ covariates or $p_{\text{sparse}} \approx 0.5$ sparsity.
  \end{itemize}

\end{document}